\pdfoutput=1             
\documentclass[12pt]{article}  

\usepackage[a4paper,margin=1in]{geometry}
\usepackage[T1]{fontenc}
\usepackage[utf8]{inputenc}
\usepackage{lmodern}     
\usepackage{microtype}   
\usepackage{amsmath,amssymb,amsthm,mathtools,mathrsfs,bm}
\usepackage{setspace}

\usepackage{graphicx}
\usepackage{booktabs}
\usepackage{multirow}
\usepackage{float}
\usepackage{caption}
\usepackage{subcaption}

\usepackage{algorithm}
\usepackage{algpseudocode}

\usepackage{amsmath}
\usepackage{enumerate}
\usepackage{url} 
\usepackage{amsthm}
\usepackage{float}
\usepackage{setspace} 
\usepackage[hidelinks]{hyperref}
\usepackage{xcolor}

\usepackage{enumitem}   

\usepackage[round]{natbib}    

\usepackage{xcolor}

\usepackage[hidelinks]{hyperref}

\newtheorem{definition}{Definition}
\newtheorem{theorem}{Theorem}[section]
\newtheorem{conjecture}{Conjecture}[section]
\newtheorem{assumption}{Assumption}

\newtheorem{corollary}{Corollary}[section]

\newtheorem{lemma}[theorem]{Lemma}

\newcommand{\tr}{\operatorname{tr}}
\newcommand{\op}{\operatorname{op}}

\newcommand{\hessianP}{\nabla^2\mathcal{P}}

\title{High-Dimensional Privacy-Utility Dynamics of Noisy\\
Stochastic Gradient Descent on Least Squares}

\author{Shurong Lin$^{1}$ \quad Eric D.~Kolaczyk$^{2}$ \quad Adam Smith$^{3}$ \quad Elliot Paquette$^{2}$\\
{\small $^{1}$Pennsylvania State University \quad $^{2}$McGill University \quad $^{3}$Boston University}\\
}

\date{}  

\begin{document}
\maketitle

\begin{abstract}
The interplay between optimization and privacy has become a central theme in privacy-preserving machine learning.
Noisy stochastic gradient descent (SGD) has emerged as a cornerstone algorithm, particularly in large-scale settings. These variants of gradient methods inject carefully calibrated noise into each update to achieve differential privacy, the gold standard notion of rigorous privacy guarantees.
Prior work primarily provides various bounds on statistical risk and privacy loss for noisy SGD, yet the \textit{exact} behavior of the process remains unclear, particularly in high-dimensional settings.
This work leverages a diffusion approach to analyze noisy SGD precisely, providing a continuous-time perspective that captures both statistical risk evolution and privacy loss dynamics in high dimensions.
Moreover, we study a variant of noisy SGD that does not require explicit knowledge of gradient sensitivity, unlike existing work that assumes or enforces sensitivity through gradient clipping.
Specifically, we focus on the least squares problem with $\ell_2$ regularization.
\end{abstract}

\section{Introduction}

Stochastic Gradient Descent (SGD) is widely recognized as the algorithm of choice for optimization in modern learning tasks, especially large-scale machine learning and deep learning. However, as models are increasingly trained on sensitive data, concerns have arisen about potential information leakage when releasing trained models.  
To address this challenge, differential privacy (DP) \citep{dwork2006calibrating} has emerged as the gold standard for formal privacy, providing quantifiable guarantees on the risks of learning from sensitive data. In practice, a standard tool for differentially private training is a noisy version of SGD, which injects calibrated noise at each iteration of classical SGD. As with any differentially private algorithm, noisy SGD protects privacy at the cost of utility. It is therefore crucial to study the privacy–utility tradeoff.  

In the DP literature, many works analyze this tradeoff by providing bounds on privacy loss or statistical risk under various conditions. Recent works often study privacy bounds of the last iterate while assuming the internal state is hidden, which yields much tighter guarantees than earlier results based on privacy composition rules. However, even when such bounds are theoretically tight, their practical relevance can be limited. For example, for $\beta$-smooth and $\lambda$-strongly convex loss functions, existing privacy and utility bounds (e.g., \cite{Bassily2014, Bassily2019, Ryffel2022DifferentialPG}) typically require relatively large $\lambda$ or small $\beta$. In other cases, results hide constants, leaving it unclear whether the bounds are meaningful in finite settings rather than only asymptotically. These limitations highlight the need for a precise theoretical analysis of both statistical risk and privacy loss, which is crucial for a deeper understanding of the privacy–utility tradeoff in noisy SGD.  

In this work, we address this gap by jointly characterizing statistical risk and privacy loss through a single diffusion-based framework. Specifically, we build on the techniques of \citet{Paquette2022HSGD, Paquette2022Implicit, Elizabeth2024}, who developed diffusion approximations for statistical risk in the non-private setting. We extend this method to the private setting, where our diffusion approximation captures not only the statistical risk of noisy SGD but also enables a precise analysis of its privacy loss.  
Our main contributions are as follows:
\begin{itemize}
    \item \textbf{Diffusion approximation} (Section 3): We propose a diffusion approximation to noisy SGD, \textit{noisy Homogenized SGD} (HSGD), whose statistical-risk trajectory and distributional behavior match those of noisy SGD.
    \vspace{-1mm}
    \item \textbf{Statistical risk trajectory} (Section 4.1): We show that the population risk concentrates around a deterministic trajectory, which we characterize exactly via Volterra integral equations. To the best of our knowledge, we are the first to precisely characterize this trajectory for noisy SGD, rather than providing upper and lower bounds as in prior work.
     \vspace{-1mm}
    \item \textbf{Privacy loss characterization} (Section 4.2): We derive the exact law of the noisy HSGD process and use it to approximate the differential privacy loss for noisy SGD under three release strategies. We analyze the risk not only at the last iterate, but also at arbitrarily chosen intermediate iterates and at the average of multiple iterates, the latter two of which have not been explored in prior work.
     \vspace{-1mm}
    \item \textbf{Clipping-free variant:} We study an unclipped variant of noisy SGD, which avoids the need for explicit sensitivity bounds or gradient clipping. Instead, privacy loss is quantified directly through the distributional law of the process.
     \vspace{-1mm}
    \item \textbf{High-dimensional settings:} Our theory applies to high-dimensional regimes where the sample size $n$ and dimension $d$ are of comparable magnitude, without assuming parameter sparsity, mimicking modern overparameterized learning problems such as wide neural networks or random feature models \citep{Wu2021LastIR,Paquette2022Implicit}.
\end{itemize}

\subsection{Related Work}
\label{sec:relatedwork}

Much of the literature on DP-SGD and DP-GD studying the privacy–utility tradeoff can be roughly categorized by an emphasis on studying either privacy or utility.

On the utility side, since the algorithms are designed based on composition rules with known privacy guarantees, earlier works concentrated on analyzing utility. They frame the problem as an optimization task, and utility bounds have been explored under various assumptions. For example, bounds have been established for Lipschitz and/or (strongly) convex objectives \citep{Bassily2014,Bassily2019,KIM2012regression}, as well as for (strongly) convex but nonsmooth objectives \citep{Feldman2020,Bassily2020}. \citet{Song2021EvadingTC} extend the discussion to non-convex generalized linear models (GLMs). More recently, \citet{Bombari2025BetterRates} provide improved rates for linear regression. While their technique is also inspired by \citet{Paquette2022Implicit}, they use two ordinary differential equations (ODEs) to obtain upper and lower bounds, in contrast to our approach, which constructs a single stochastic differential equation (SDE) with a matching rate.  

On the privacy side, recent works study tight privacy bounds under the assumption that only the last iterate is revealed. Classical DP composition rules are notoriously conservative, giving privacy bounds that grow unbounded with the number of iterations. By assuming the internal states are hidden, it has been shown that privacy loss may converge. \citet{chourasia2021differential,Ye2022DifferentiallyPL,Ryffel2022DifferentialPG} study the strongly convex and smooth case, while \citet{Altschuler2022PrivacyON} relax strong convexity to convexity, and \citet{asoodeh2023privacy} extend the discussion to the non-convex setting. 
Since these privacy bounds are tighter, the tradeoff is improved: for the same level of DP, better utility can be achieved using the analyses developed in utility-focused work.

\section{Background}
We review background on Rényi DP, a notion of DP that is widely used for obtaining tight privacy guarantees in previous analysis of DP-SGD. For classical variants, such as $\varepsilon$-DP and $(\varepsilon,\delta)$-DP, we refer the reader to \citet{dwork2006calibrating,dwork2014algorithmic}.

Let $D$ and $D'$ be datasets of the same size. We call them \emph{neighboring} datasets if they differ in exactly one record. 
In this work, we refer to the differing records as a \emph{neighboring pair}.
\begin{definition}[{\citep{Mironov2017RnyiDP}} Rényi Differential Privacy]
\label{def:rdp}
Let $\alpha > 1$. A randomized algorithm $\mathcal{A} : \mathcal{X}^n \to \mathbb{R}^d$ satisfies $(\alpha,\varepsilon)$-\emph{Rényi Differential Privacy (RDP)} if for any two neighboring datasets $D, D' \in \mathcal{X}^n$, the $\alpha$-Rényi divergence satisfies
\[
R_\alpha(\mathcal{A}(D)\|\mathcal{A}(D')) \le \varepsilon.
\]
For a pair of measures $P, Q$ over the same space where $P$ is absolutely continuous with respect to $Q$, the $\alpha$-Rényi divergence is defined as
\[
R_\alpha(P \| Q) = \frac{1}{\alpha-1} \ln \int \left(\frac{dP}{dQ}\right)^\alpha dQ.
\]
\end{definition}
We refer to $R_\alpha(\mathcal{A}(D)\|\mathcal{A}(D'))$ as the \emph{Rényi-DP loss} of algorithm $\mathcal{A}$ on datasets $D,D'$.

\paragraph{Problem Setup}
Let $(\mathbf{a}, b)$ denote a data point, where $\mathbf{a} \in \mathbb{R}^d$ is the feature vector and $b \in \mathbb{R}$ is the corresponding label.
Let $\bm x \in \mathbb{R}^d$ be the parameter vector to be estimated.
We consider the $\ell_2$-regularized least squares problem
\[
\min_{\bm x} f(\bm x; \bm a, b) := \tfrac{1}{2}\left( (\bm a^\top \bm x - b)^2 + \delta \|\bm x\|^2 \right),
\]
solved by noisy SGD described in Algorithm~\ref{alg:dp_sgd}, which we state in terms of
the one-pass variant where data points are visited only once. 
One-pass DP-SGD forms the foundation of streaming and online learning for large-scale data and runs significantly faster. It avoids repeated use of sensitive records and offers advantages from a privacy perspective. It also provides a cleaner analytical baseline, while still being extendable to the multi-pass setting.
In practice, the order of the data points might be shuffled once before training.

\begin{algorithm}
\caption{One-Pass Noisy Stochastic Gradient Descent (SGD)}
\label{alg:dp_sgd}
\begin{algorithmic}[1]
\Require Loss function $f(\bm x)$, dataset $D=\{(\bm a_i, b_i)\}_{i=1}^{n}$, initial parameter $x_0$, learning rate schedule $\eta_k$, noise scale $\sigma$.
\For{$k = 0$ to $n-1$}
    \State Compute stochastic gradient:
    $\bm g_k \gets \nabla_{\bm x} f(\bm x_{k}; \bm a_k, b_k) + \sigma \bm  Z_k$, where $\bm Z_k \sim \mathcal{N}(0, I_d)$.
    \State Update parameter: $\bm x_{k+1} \gets \bm x_{k} - \eta_k \bm g_k$.
\EndFor
\State \Return $\bm x_n$.
\end{algorithmic}
\end{algorithm}

We focus on the high-dimensional setting where the sample size $n$ and the dimension $d$ are both large.
For our analysis to hold, the data and the learning rate regime must be well-conditioned in the presence of large $d$. 
Formally, we need the following assumptions. 

\begin{assumption}
\label{assmpt:covariance}
The spectral norm of $\bm \Sigma:= \mathbb{E}\bm a\bm a^\top$ is bounded independent of $d$.
\end{assumption} 

\begin{assumption}
\label{assmpt:subgaus}
Let  $\|\cdot\|_{\psi_p} $ for $p\geq1$ denote the Orlicz norms \citep{Vershynin} of a random variable $w$ which is defined as $\|w\|_{\psi_p} = \inf \{t: \mathbb{E}e^{|X|^p/t^p} \leq 2\} $.
The distribution of $b $ follows
\[
b = \bm a^\top \tilde{\bm x} + w,
\]
where $w $ is sub-Gaussian with zero mean and satisfies $\|w\|_{\psi_2} \leq d^{{c_*}} $ and 
the ground truth satisfies $\|\tilde{\bm x}\| \leq d^{c_*} $ for some small constant $c_* > 0 $.
\end{assumption} 


\section{Diffusion Approximation: Noisy Homogenized SGD}
We define the regularized population risk
$$
\mathcal{R}(\bm x) := \tfrac{1}{2}\mathbb{E}_{(\bm a, b)}\left(\bm a^\top \bm x - b\right)^2 
+ \tfrac{\delta}{2}\|\bm x\|^2,
$$
which is the expected risk of the loss function $f(\bm x)$. It governs the expected dynamics of the SGD iterates, thus the update in Algorithm~\ref{alg:dp_sgd} can be decomposed as
\[
\bm x_{k+1} 
= \bm x_k - \eta_k\left(\nabla\mathcal{R}(\bm x_k) + \Delta \mathcal{M}_k + \sigma \bm Z_k\right),
\]
where $\Delta \mathcal{M}_k$ is a martingale increment capturing the stochasticity from sampling,  
and $\bm Z_k \sim \mathcal{N}(0,I_d)$ represents the injected Gaussian noise for DP purposes.
This recursion naturally separates into a deterministic drift term and stochastic fluctuations arising from both sampling and DP. 
Let $\gamma:[0,\infty)\to[0,\infty)$ denote a smooth, bounded learning rate schedule in continuous time. 
We relate this to the discrete learning rates $\{\eta_k\}$ via the scaling
\[
\eta_k = \frac{\gamma(k/d)}{d}, \qquad t = k/d,
\]
so that each discrete update corresponds to an increment of size $1/d$ in continuous time. 
The range of iteration indices corresponds to the continuous-time interval $[0,T]$, where $T = n/d$.
Under this scaling, the noisy SGD dynamics are approximated by the following stochastic differential equation (SDE):
\begin{equation}
\label{eq:SDE}
\begin{aligned}
d \bm X_t 
&= -\gamma(t)\nabla \mathcal{R}(\bm X_t) \mathrm{d}t + \gamma(t)\sqrt{\tfrac{1}{d}\left( 2\mathcal{P}(\bm X_t)\bm\Sigma + \sigma^2 I_d \right)} \mathrm{d}\bm B_t,
\end{aligned}
\end{equation}
where $\mathcal{P}(\bm x) := \frac{1}{2} \mathbb{E}_{(\bm a, b)}(\bm a^\top \bm x - b)^2 $ is the unregularized population risk and  $\bm{B}_t $ is a $d $-dimensional Brownian motion.
We define the resulting continuous-time process as \emph{noisy homogenized SGD (HSGD)} with initialization $\bm X_0 = \bm x_0$. 

Our theory shows that, under this correspondence, quadratic population risks (defined below) evaluated on $\bm x_k$ and $\bm X_{k/d}$ are close in high dimensions, and the discrepancy vanishes as $d\to\infty$.
\begin{definition}[Quadratic population risk]
\label{def:quadrisk}
    The population risk $q: \mathbb{R}^d \to \mathbb{R}$ is quadratic if it can be written as $q = \frac{1}{2}\bm x^\top \bm T\bm x + \bm u^\top x + c$ for some $d\times d$ symmetric matrix $\bm T$, vector $u\in\mathbb{R}^d$ , and scalar $c\in\mathbb{R}$.
    We define the $\|\cdot\|_{C^2}$-norm of $q$ as
\begin{equation*}
    \|q\|_{C^2} := \|\nabla^2q\| + \|\nabla q(0)\| + |q(0)| =  \|T\| + \|\bm u\| + |c|, 
\end{equation*}
\end{definition}
We use the phrase \textit{with overwhelming probability (w.o.p.)} to mean that a statement holds 
except on an event whose probability is at most $e^{-\omega(\ln d)}$, 
where $\omega(\ln d)$ grows faster than $\ln d$ as $d \to \infty$. 
Equivalently, the failure probability decays faster than any polynomial in $d$.
\begin{theorem}[Quadratic Risk Equivalence]
\label{thm:comparison}
For any quadratic loss function $q $ with bounded $\|\cdot\|_{C^2}$-norm, there exist a constant $c_0$ such that, for any $n\leq d\ln d/c_0$, w.o.p. we have,
\begin{equation}
\label{eq:comparison-dp}
\begin{aligned}
& \sup_{0 \le k \le n} \Big|q(\bm x_k) - q(\bm X_{k/d})\Big|
\le\|q\|_{C^2} e^{c_0 n/(8d)} \Big(
    d^{-\tfrac12+9c_*}
    + 2\sigma d^{-\tfrac12+2c_*} 
    + \sigma^2d^{-\tfrac12}
\Big),
\end{aligned}
\end{equation}
where $c_*$ is defined in Assumption~\ref{assmpt:subgaus}.
\end{theorem}
In the theorem, the small constant $c_*$ (defined in Assumption~\ref{assmpt:subgaus}) must be less than $1/24$ for the bound to vanish as $d \to \infty$. This threshold can be improved, although we do not pursue such refinements here.
The terms involving the noise scale $\sigma$ become negligible provided $\sigma = o(d^{9c_*/2})$, since they are dominated by 
the first term $d^{-1/2+9c_*}$.
Thus, for moderate noise levels, the risks of noisy SGD and noisy HSGD are indistinguishable in high dimensions.

Given that noisy SGD and noisy HSGD processes, $\bm x_k$ and $\bm X_t$, exhibit the same risk behavior in high dimensions, in Section~3 we focus on analyzing noisy HSGD and characterizing its evolution over time. Note that the two processes are independent, which implies that the risk of both noisy SGD and noisy HSGD concentrates around a deterministic trajectory, for which we derive a closed-form expression. 
{\color{black}
Moreover, we expect the agreement between $\bm x_k$ and $\bm X_t$ to extend beyond the quadratic risk (a first-order property) to the fluctuations of quadratic functions of the iterates (a second-order property). In the least-squares setting, both noisy SGD and noisy HSGD are linear–Gaussian systems, so agreement at the second-order level implies that their distributions are aligned. 
Consistent with this intuition, our simulations provide empirical evidence that noisy SGD and noisy HSGD exhibit closely matching finite-dimensional distributions (see Figure~\ref{fig:qq}), which motivates the following conjecture.
\begin{conjecture}[Distributional Equivalence]
\label{conj:distribution}
In the quadratic-loss setting, the finite-dimensional distributions of noisy SGD and noisy HSGD converge as $d \to \infty$. That is, for any fixed collection of time points $0 \le k_1 < \cdots < k_J \le n$, the joint law of 
$(\bm x_{k_j})_{j=1}^J$ and $(\bm X_{k_j/d})_{j=1}^J$
become indistinguishable in the high-dimensional limit.
\end{conjecture}
This distributional equivalence implies that the privacy loss of noisy SGD can be faithfully analyzed through its diffusion surrogate, noisy HSGD.
}

\section{Privacy-Utility Tradeoffs}
Using the equivalence between noisy HSGD and noisy SGD (Section~2), we analyze the statistical risk and privacy loss of noisy HSGD to characterize those of noisy SGD.
\subsection{Statistical Risk} 
We first show that the risk curve of $\bm X_t$ concentrates around its expectation, which defines a deterministic trajectory. We then provide a precise characterization of this trajectory.

\begin{theorem}[Risk Concentration]
\label{thm:concen_risk}
For any constant $c_1>0$, there exist constants $c_0, c_2>0$ such that, 
for any $n \le d \ln d / c_0$ and $q \in \{\mathcal{P}, \mathcal{R}\}$, 
the following holds with probability at least $1 - c_2 d^{-c_1}$:
\begin{equation}
    \sup_{0 \le t \le n/d} 
    \big| q(\bm X_t) - \mathbb{E}[q(\bm X_t)] \big|
     \le d^{-1/2 + 2c_*}.
\end{equation}
Equivalently, expressing time in discrete SGD indexing ($t = k/d$), we have
\begin{equation}
    \sup_{0 \le k \le n} 
    \big| q(\bm X_{k/d}) - \mathbb{E}[q(\bm X_{k/d})] \big|
     \le d^{-1/2 + 2c_*}.
\end{equation}
\end{theorem}

Since the discrete process SGD $\bm x_k$ and the continuous HSGD
$\bm X_{k/d}$  are independent, the agreement of their risk trajectories indicates a form of concentration. In particular, $\bm x_k$  concentrates around the same deterministic curve given by the expected trajectory of $\bm X_{k/d}$, despite the two processes being driven by independent randomness. This highlights that high-dimensional noisy SGD behaves in a stable and predictable manner when viewed through its diffusion approximation.
 
To analyze the expected risks, we first consider the \textit{gradient flow (GF)}, denoted as $X^{\mathrm{gf}}_{t}$, which solves the ordinary differential equation (ODE):
\begin{equation*}
    d\bm X^{\mathrm{gf}}_{t} = -\nabla \mathcal{R}(\bm X^{\mathrm{gf}}_{t})\mathrm{d}t, \quad \bm X^{\mathrm{gf}}_{0} = \bm x_0.
\end{equation*}
We rescale the GF to account for a learning rate schedule $\gamma(t)$. 
Recall that $\tilde{\bm x}$ denotes the ground truth parameter. Let $\Gamma(t)=\int_0^t \gamma(s)\mathrm{d}s$, then
\begin{equation}
\label{eq:ODE}
    \begin{aligned}
        d\bm X^{\mathrm{gf}}_{\Gamma(t)} = -\gamma(t)\nabla \mathcal{R}(\bm X^{\mathrm{gf}}_{\Gamma(t)})\mathrm{d}t, \quad \bm X^{\mathrm{gf}}_{\Gamma(0)} = \bm x_0.
    \end{aligned}
\end{equation}
Let $\bm A = \bm \Sigma + \delta \bm I_d$, 
and 
\[
\Phi(t) = e^{-\bm A\Gamma(t)},   \Phi(t,u)= \Phi(t)\Phi^{-1}(u)
\]
The solution of (\ref{eq:ODE}) is given by
\begin{equation*}
\begin{aligned}
\bm{X}^{\mathrm{gf}}_{\Gamma(t)} & =\Phi(t) \bm{x}_0  + \int_0^t  \Phi(t,u) \gamma(u)  \bm{\Sigma}\tilde{\bm x}  du.
    \end{aligned}
\end{equation*}
In the case of noisy HSGD, setting the Brownian noise term to zero yields precisely the gradient flow $X^{\mathrm{gf}}_{\Gamma(t)}$. Hence, gradient flow serves as the deterministic baseline for analyzing the risk of noisy HSGD. 

\begin{theorem}[Risk Trajectory]
\label{thm:volterra}
Let  $P_t \stackrel{\text{def}}{=} \mathbb{E}[\mathcal{P}(\bm X_t)]$ and $R_t \stackrel{\text{def}}{=} \mathbb{E}[\mathcal{R}(\bm X_t)]$ and
\begin{equation*}
\begin{aligned}
    &G(t,s;\bm M) = \frac{\gamma^2(s)}{d}\tr 
    \left( 
    (\bm \Sigma\bm M \Phi^2(t,s) \right);\\
    &G'(t,s;\bm M) = \frac{\sigma^2\gamma^2(s)}{2d}\tr 
    \left(  \bm M \Phi^2(t,s) \right).
\end{aligned}
\end{equation*}
Then, $P_t$ and $R_t$ are characterized by the Volterra equations:
\begin{equation}
\begin{aligned}
     P_t  = \mathcal{P}(\bm {X}^{\mathrm{gf}}_{\Gamma(t)}) &+ \int_0^t G(t,s;\hessianP)P_s \mathrm{d}s  + \int_0^t G'(t,s;\nabla^2\mathcal{P})\mathrm{d}s;
\end{aligned}
\end{equation}
\begin{equation}
\begin{aligned}
    R_t = \mathcal{R}(\bm {X}^{\mathrm{gf}}_{\Gamma(t)}) & + \int_0^t G(t,s;\nabla^2\mathcal{R})P_s \mathrm{d}s + \int_0^t G'(t,s;\nabla^2\mathcal{R})\mathrm{d}s.
\end{aligned}
\end{equation}
\end{theorem}
The characterization of the expected risk trajectory provides a precise decomposition of how noisy HSGD evolves. Specifically, the Volterra integral equations entail contributions of three components: (i) the deterministic risk of gradient flow, (ii) the additional fluctuations from sampling, and (iii) the effect of injected DP noise. This decomposition clarifies how each source of randomness shapes the overall dynamics. In special cases, such as constant learning rate and isotropic covariance, closed-form analytical solutions are available.

{\color{black}
Combining Theorems~\ref{thm:comparison} and \ref{thm:concen_risk}, we obtain the following corollary, which shows that the risk of noisy SGD follows the deterministic trajectory given in Theorem~\ref{thm:volterra}.
\begin{corollary}
\label{cor:sgd-risk}
For any constant $c_1>0$, there exists constants $c_0,c_2>0$ such that, for any $n\leq d \ln d/c_0$ and $q \in \{\mathcal{P}, \mathcal{R}\}$, the following holds with probability at least $1 - c_2 d^{-c_1}$:
\begin{equation}
\begin{aligned}
& \sup_{0 \le k \le n}
 \Big| q(\bm x_k) - \mathbb{E}[q(\bm X_{k/d})] \Big|
 \le  2\|q\|_{C^2}\, e^{c_0 n/(8d)}
 \Big(d^{-\tfrac12+9c_*} + \sigma d^{-\tfrac12+2c_*} + \sigma^2 d^{-\tfrac12}\Big).
\end{aligned}
\end{equation}
\end{corollary}
}
This shows that the risk of noisy SGD is well-approximated by a deterministic curve, rather than fluctuating in an uncontrolled way. Importantly, this trajectory can be explicitly characterized, offering a much sharper description of its behavior over time. To the best of our knowledge, our work is the first to provide such a precise characterization of the statistical risk for noisy SGD, whereas prior analyses have typically been limited to upper and lower bounds.

\subsection{Privacy Loss}
Since $\mathcal{P}(\bm X_t)$ concentrates around $P_t$, we approximate it by $P_t$, which leads to a linear SDE:
\begin{equation}
\label{eq:SDE2}
    d \bm X_t = -\gamma(t) \nabla \mathcal{R}(\bm X_t) \mathrm{d}t + \gamma(t)\bm Q^{1/2}(t) d \bm B_t,
\end{equation}
where $\bm Q(t) = \frac{1}{d}\left(2 P_t \bm \Sigma + \sigma^2 \bm I_d\right).$
The solution to (\ref{eq:SDE2}) follows 
\begin{equation}
\label{eq:HSGD_dist}
\bm{X}_t \sim \mathcal{N}(m(t), V(t)),
\end{equation}
where the mean and covariance are given by
\begin{equation*}
\label{eq:mean_standard}
    m(t) 
    =\Phi(t)\bm x_0 
      + \Phi(t)\int_0^t \Phi^{-1}(u)\gamma(u)\bm \Sigma\tilde{\bm x}du;
\end{equation*}
\begin{equation*}
\label{eq:cov_solution}
\begin{aligned}
    V(t) = \int_0^t \Phi(t,u)\gamma^2(u)\bm Q(u)\Phi^\top(t,u)du .
\end{aligned}
\end{equation*}
This characterizes the law of noisy HSGD. It serves as the foundation for introducing a surgical adjustment to define two neighboring processes described in the following.

\paragraph{Neighboring Noisy HSGD}
To study the privacy loss of noisy SGD, 
we couple two HSGD processes corresponding to running SGD on neighboring datasets. 
Suppose the neighboring datasets differ in exactly one datapoint, denoted 
$(\bm a, b)$ versus $(\bm a', b')$, and that it is used in the $\ell$-th iteration. 
The one-step SGD update at iteration $\ell$ is
$$\bm x_{l+1} \gets \bm x_{l} - \eta_{\ell}\big((\bm a\bm a^\top+\delta\bm I_d)\bm x - b\bm a\big) 
- \eta_{\ell} \sigma \bm Z.$$
This is the update where the two neighboring trajectories diverge, 
and we refer to it as the \emph{differentiating update}. 
To model this discrete update within HSGD, we embed the $\ell$-th iteration at the continuous time $s=\ell/d>0$ and use the rescaled learning rate $\eta_{\ell}=\gamma(s)/d$.
We denote by $\bm X_{s^{-}}$ and $\bm X_{s^{+}}$ the random variables immediately
\textit{before} and \textit{after} the differentiating update at time $s$.
Before the update, the SDE solution yields the Gaussian law
\begin{equation}
    \label{eq:befores}
    \bm X_{s^-}\sim \mathcal N\big(m(s),V(s)\big).
\end{equation}
After the update, using $(\bm a,b)$, 
the state evolves according to the affine map
\begin{equation}
\label{eq:jump_general}
\bm X_{s^+}^{(1)}  =  \bm C\bm X_{s^-} + \bm c  +  \bm e,
\end{equation}
where $
\bm e \sim \mathcal N(\bm 0,\frac{\gamma^2(s)}{d^2} \sigma^2 \bm I_d)$ and
\begin{equation*}
\bm C_1 = \bm I_d - \frac{\gamma(s)}{d}(\bm a\bm a^\top + \delta \bm I_d),
\quad
\bm c_1 = \frac{\gamma(s)}{d} b\bm a.
\end{equation*}
Here $\bm C_1$ encodes the linear contraction induced by the gradient step, $\bm c_1$ is the deterministic shift coming from the label $b$, and $\bm e$ captures the injected Gaussian noise. 

For the second trajectory, using $(\bm a',b')$, define analogously
\begin{equation*}
\bm C_2 = \bm I_d - \frac{\gamma(s)}{d}(\bm a'\bm a'^\top + \delta \bm I_d),  
\bm c_2 = \frac{\gamma(s)}{d} b'\bm a'.
\end{equation*}
Because \eqref{eq:jump_general} is an affine transformation to Gaussian, we have
\begin{equation*}
\begin{aligned}
\label{eq:postjump}
&\bm X_{s^+}^{(i)} \sim \mathcal N\big(m_i(s^+),V_i(s^+)\big),\quad i\in\{1,2\},\\
& m_i(s^+) = \bm C_im(s) + \bm c_i,\\
&V_i(s^+) = \bm C_iV(s)\bm C_i^\top + \frac{\gamma^2(s)}{d^2} \sigma^2 \bm I_d.
\end{aligned}
\end{equation*}

Conditioned on the differentiating update at time $s$,
we treat the post-differentiation  state $X_{s^+}$ as a new initial condition and evolve under the same SDE
\eqref{eq:SDE2} for any $t>s$. 
\begin{equation}\label{eq:after_s}
\begin{aligned}
\bm X_t^{(i)} &\sim \mathcal N\!\big(m_i(t;s),V_i(t;s)\big),\quad i\in\{1,2\},\\
m_i(t;s) & = \Phi(t,s)m_i(s^+) 
        + \int_s^t \Phi(t,u)\gamma(u)\bm\Sigma\tilde{\bm x}du,\\
V_i(t;s) &= \Phi(t,s)V_i(s^+)\Phi(t,s)^\top + \int_s^t \Phi(t,u)\gamma^2(u)\bm Q(u)\Phi^\top(t,u)du .
\end{aligned}
\end{equation}
Here, the matrix $\Phi(t,s)$ is the \textit{state–transition} matrix, which propagates the state at time $s^{+}$ to time $t>s$ under the linear dynamics. Since after the differentiation both processes evolve under the same law in (\ref{eq:SDE}) again, the only remaining discrepancy is the mismatch at $s$ transported by the transition matrix. Consequently, the mean and covariance differences of two neighboring HSGDs lie in the first terms and decay deterministically at rate $\Phi(t,s)=e^{-\bm A((\Gamma(t)-(\Gamma(s))}$ as $t$ increases. 
\begin{equation*}
\begin{aligned}
    &m_1(t;s)-m_2(t;s)=\Phi(t,s)\big(m_1(s^+)-m_2(s^+)\big),\\
&V_1(t;s)-V_2(t;s)=\Phi(t,s)\big(V_1(s^+)-V_2(s^+)\big)\Phi(t,s)^\top
\end{aligned}
\end{equation*}

\paragraph{Rényi-DP Loss} 
{\color{black}
Let $D,D'$ be two neighboring datasets differing in records $(\bm a,b)$ and $(\bm a',b')$. 
Conditioned on $D$ and $D'$, we model the high-dimensional privacy loss of noisy SGD by that of noisy HSGD, which reduces to analyzing the divergence between two neighboring HSGD laws induced by the pair $(\bm a,b)$ and $(\bm a',b')$. 
Since all other datapoints are identical across $D$ and $D'$, the only source of discrepancy is the single differentiating update, represented in HSGD as the affine map at the differentiating update. 
The contribution of the common datapoints is absorbed into the shared Gaussian law $\mathcal N(m(t),V(t))$, so the privacy loss depends solely on how the differentiating update perturbs the mean and covariance at $s$ and how this perturbation propagates forward in time.

\paragraph{Remark on Approximation vs. SGD}
Our privacy analysis in the following is carried out for noisy HSGD.
By Conjecture~\ref{conj:distribution} supported by empirical evidence, noisy SGD and noisy HSGD have closely matching finite-dimensional distributions in high dimensions. 
Consequently, the Rényi divergences we derive for HSGD serve as accurate approximations of those for noisy SGD. 
In particular, the results of Section~4 should be interpreted as high-dimensional privacy guarantees for noisy SGD, obtained via its diffusion limit.

}

In shuffled SGD the index $\ell$ of the differing record is chosen uniformly at random from $\{1,\dots,n\}$. 
In the HSGD scaling this corresponds to a random update time $s \sim \mathrm{Unif}(0,T)$. 
Conditional on $s$, the two neighboring processes have Gaussian laws $\mathcal N\!\big(m_i(t;s),V_i(t;s)\big)$. 
As a result,
the marginal law at time $t$ is a mixture over $s$.
For the case where the last iterate is released, define the mean difference and the $\alpha$-mixture covariance, where $\alpha$ denotes the Rényi order:
\begin{equation*}
\begin{aligned}
    \Delta(t;s)& :=m_1(t;s)-m_2(t;s),\\
    M_\alpha(t;s)& :=\alpha V_1(t;s)+(1-\alpha)V_2(t;s). 
\end{aligned}
\end{equation*}

\begin{theorem}[RDP for last iterate]
\label{thm:rdp-last}
Fix $\alpha>1$ and $t\in(0,T]$. Let $\mathcal C$ be the possible set of
neighboring pairs $\theta=((\bm a,b),(\bm a',b'))$. For each $\theta \in \mathcal{C}$ and $s > t$, the divergence $D_\alpha(t; s, \theta) = 0$.  
For $s \le t$, the Rényi divergence of order $\alpha$ at time $t$ is given by
\begin{equation}
\label{eq:Gauss-Rényi}
\begin{aligned}
& D_\alpha(t;s,\theta)
=\frac{\alpha}{2}\Delta(t;s)^\top M_\alpha(t;s)^{-1}\Delta(t;s)\\
&\quad+\frac{1}{2(\alpha-1)}\ln\!\left(
\frac{\det V_1(t;s)^{\alpha}\det V_2(t;s)^{1-\alpha}}{\det M_\alpha(t;s)}
\right).
\end{aligned}
\end{equation}

With $s\sim\mathrm{Unif}(0,T)$, corresponding to randomly shuffled SGD, the order-$\alpha$ Rényi--DP \emph{loss} at time $t$, defined as the
worst case over neighboring pairs, is given by
\begin{equation}\label{eq:RDP-last}
\begin{aligned}
& \varepsilon_\alpha(t)
\leq \frac{1}{\alpha-1}\ln\!\Big(
\frac{T-t}{T}
+ \frac{1}{T}\int_0^{t}
\exp\!\big((\alpha-1)\sup_{\theta\in\mathcal C} D_\alpha(t;s,\theta)\big)ds
\Big).
\end{aligned}
\end{equation}

\end{theorem}

Since we fully characterize the evolution of neighboring HSGD processes, we can also quantify the privacy loss under two other release strategies that may arise in DP when publishing only the last iterate is insufficient: releasing multiple iterates and releasing their average.

Let $0<t_1<\cdots<t_J\le T$ and release a collection of $J$ states
$
\mathcal X_{\bm t}
:=\big(\bm X_{t_1}^\top,\ldots,\bm X_{t_J}^\top\big)^\top\in\mathbb R^{Jd}.
$
Conditioned on the differentiation time $s$ and neighboring pair
$\theta=((\bm a,b),(\bm a',b'))$, each process is Gaussian:
\begin{equation}\label{eq:block-law}
\mathcal X_{\bm t}^{(i)}\sim
\mathcal N\!\big(\bm m_i(\bm t;s),\bm V_i(\bm t;s)\big),\qquad i\in\{1,2\},
\end{equation}
where the stacked mean is
\begin{equation}\label{eq:block-mean}
\bm m_i(\bm t;s)
:=\big(m_i(t_1;s)^\top,\ldots,m_i(t_J;s)^\top\big)^\top ,
\end{equation}
and the $Jd\times Jd$ \emph{block covariance} $\bm V_i(\bm t;s)$ has
$d\times d$ blocks
\begin{equation}
\begin{aligned}
\label{eq:block-cov}
& \big[\bm V_i(\bm t;s)\big]_{rj}
 =\operatorname{Cov}\!\big(\bm X_{t_r},\bm X_{t_j}\big) =\begin{cases}
\Phi(t_r,t_j)V_i(t_j;s), & r\geq j,\\[4pt]
\big(\Phi(t_j,t_r)V_i(t_r;s)\big)^\top, & r<j,
\end{cases}
\end{aligned}
\end{equation}
so the diagonal blocks $\big[\bm V_i(\bm t;s)\big]_{rr}=V_i(t_r;s)$ are variances, and the off-diagonal blocks are cross-covariances.

Define the stacked mean difference and the $\alpha$-mixture covariance
\begin{equation}
\label{eq:stacked-diff-mix}
\begin{aligned}
\Delta^{[J]}(\bm t;s)&:=\bm m_1(\bm t;s)-\bm m_2(\bm t;s),\\
\bm M_\alpha^{[J]}(\bm t;s)&:=\alpha\bm V_1(\bm t;s)+(1-\alpha)\bm V_2(\bm t;s).
\end{aligned}
\end{equation}

\begin{theorem}(RDP for releasing $J$ iterates)
\label{thm:rdp-multi}
Let $\bm t=(t_1,\ldots,t_J)$ with $t_J\le T$ be the $J$ time points at which the states are released.
For any neighboring $\theta\in\mathcal C$ and $s\le t_J$, the Rényi divergence is
\begin{equation}\label{eq:Gauss-Rényi-j}
\begin{aligned}
& D_\alpha^{[J]}(\bm t;s,\theta)
=\frac{\alpha}{2}\Delta^{[J]}(\bm t;s)^\top
\big(\bm M_\alpha^{[J]}(\bm t;s)\big)^{-1}\Delta^{[J]}(\bm t;s)\\
&+\frac{1}{2(\alpha-1)}\ln\!\left(
\frac{\det\!\big(\bm V_1(\bm t;s)\big)^{\alpha}
      \det\!\big(\bm V_2(\bm t;s)\big)^{1-\alpha}}
     {\det\!\big(\bm M_\alpha^{[J]}(\bm t;s)\big)}
\right).
\end{aligned}
\end{equation}
If $s\in(t_{\ell-1},t_\ell]$ (with $t_0:=0$), then the first $\ell-1$ blocks of
$\mathcal X_{\bm t}$ coincide across the two processes, so
$D_\alpha^{[J]}(\bm t;s,\theta)$ equals \eqref{eq:Gauss-Rényi-j} computed on the
suffix $(t_\ell,\ldots,t_J)$. 
With $s\sim\mathrm{Unif}(0,T)$, the order-$\alpha$
Rényi–DP loss for releasing $\mathcal X_{\bm t}$ is given by
\begin{equation}\label{eq:RDP-j}
\begin{aligned}
&\varepsilon_\alpha^{[J]}(\bm t) \leq
\frac{1}{\alpha-1}\ln\!\Big(
\frac{T-t_J}{T}\frac{1}{T}\int_0^{t_J}
\exp\!\Big((\alpha-1)\sup_{\theta\in\mathcal C}
D_\alpha^{[J]}(\bm t;s,\theta)\Big)ds
\Big).
\end{aligned}
\end{equation}
\end{theorem} 

With $\bm t=(t_1,\ldots,t_J)$ as above, release
$$\bar{\bm X}_{\bm t}:=\frac{1}{J}\sum_{j=1}^J \bm X_{t_j}.$$
Conditioned on $s$ and $\theta$, $\bar{\bm X}_{\bm t}^{(i)}\sim
\mathcal N(\bar m_i,\bar V_i)$ with
\begin{equation*}
    \begin{aligned}
        \bar m_i:=\frac{1}{J}\sum_{j=1}^J m_i(t_j;s),  
        \bar V_i:=\frac{1}{J^2}\sum_{r,j=1}^J \Sigma_i(t_r,t_j;s),
    \end{aligned}
\end{equation*}
where $\Sigma_i(t_r,t_j;s)=\operatorname{Cov}(\bm X_{t_r},\bm X_{t_k})$ as in \eqref{eq:block-cov}.
Define 
\begin{equation}
\begin{aligned}
    \Delta^{\mathrm{avg}}&:=\bar m_1-\bar m_2\\
    \bar M_\alpha&:=\alpha\bar V_1+(1-\alpha)\bar V_2
    \end{aligned}
\end{equation}
\begin{theorem}[RDP for releasing the average]
\label{thm:rdp-avg}
Let $\bm t=(t_1,\ldots,t_J)$ with $t_J\le T$ be the $J$ time points at which the states are averaged and then released.
For any neighboring $\theta\in\mathcal C$ and $s\le t_J$, the Rényi divergence is
\begin{equation}
\label{eq:Gauss-Rényi-avg-min}
\begin{aligned}
&D_\alpha^{\mathrm{avg}}(\bm t;s,\theta)
=\frac{\alpha}{2}\Delta^{\mathrm{avg}\top}\bar M_\alpha^{-1}\Delta^{\mathrm{avg}}
  +\frac{1}{2(\alpha-1)}\ln\!\left(
\frac{\det(\bar V_1)^{\alpha}\det(\bar V_2)^{1-\alpha}}
{\det(\bar M_\alpha)}
\right).
\end{aligned}
\end{equation}
With $s\sim\mathrm{Unif}(0,T)$,
\begin{equation}
\begin{aligned}
    & \varepsilon_\alpha^{\mathrm{avg}}(\bm t)
\leq \frac{1}{\alpha-1}\ln\Big(
\frac{T-t_J}{T}
  +\frac{1}{T}\int_0^{t_J}
\exp\!\Big((\alpha-1)\sup_{\theta\in\mathcal C}
D_\alpha^{\mathrm{avg}}(\bm t;s,\theta)\Big)ds
\Big).
\end{aligned}
\end{equation}

\end{theorem}

\section{Experiments}
We empirically evaluate the diffusion approximation of noisy SGD via noisy HSGD by comparing their dynamics through simulations.

\paragraph{Simulation Setup}
We simulate a high-dimensional linear model with $d=1000$ features and a total of $n=1500$ samples. The ground truth is $\tilde{\bm x}\sim\mathrm{Unif}(0,\bm I_d/\sqrt d)$; the $n\times d$ design matrix has i.i.d.\ entries $a_{ij}\sim\mathrm{Unif}(0,1/\sqrt d)$; labels follow $b_i=\bm a_i^\top\tilde{\bm x}+\xi_i$ with $\xi_i\sim\mathcal N(0,0.01/d)$ (clipped at $\pm 3$ standard deviations to ensure boundedness for privacy).
SGD with constant learning rate $\eta=0.05$ and the DP noise scale $\sigma\in\{1,1.25,1.5\}$, initialized at $\bm x_0\sim\mathcal N(\bm 0,\bm I_d)$ was applied to the linear
regression problem with a regularization parameter of $0.1$.


\begin{figure}[h]
  \centering
  \begin{subfigure}[b]{0.35\columnwidth}
    \centering
    \includegraphics[width=\linewidth]{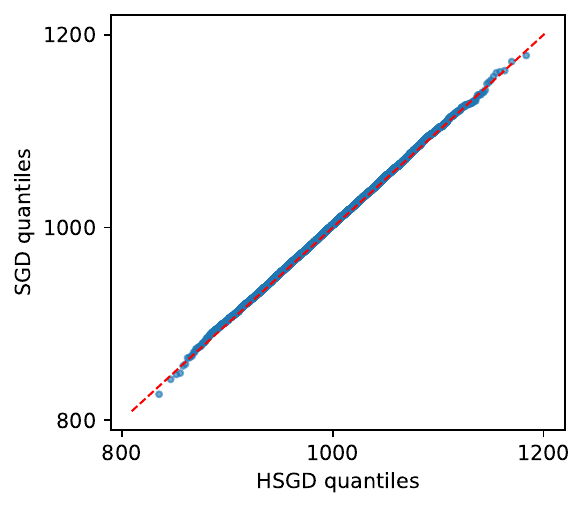}
    \caption{QQ plot}
    \label{fig:qq}
  \end{subfigure}%
  \begin{subfigure}[b]{0.35\columnwidth}
    \centering
    \includegraphics[width=\linewidth]{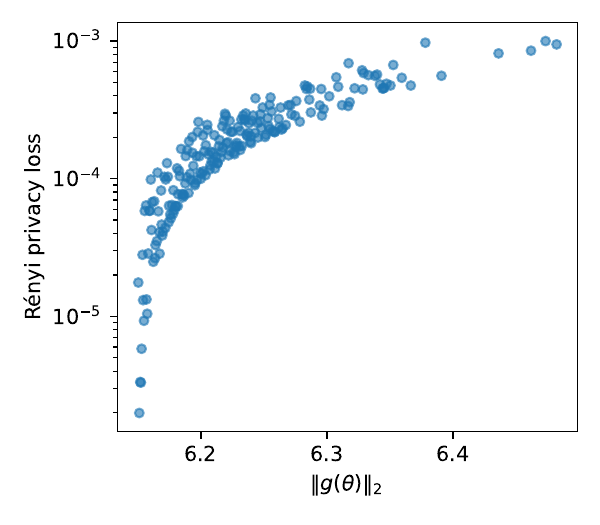}
    \caption{Privacy vs. $\|g{(\theta)}\|_2$}
    \label{fig:trend}
  \end{subfigure}
  \caption{Diagnostics of HSGD approximation for privacy loss ($\sigma=1.5$).}
  \label{fig:misc}
\end{figure}

\paragraph{Privacy Loss Approximation}
As discussed in Conjecture~1, we expect that noisy SGD and noisy HSGD have closely matching distributions. Figure~\ref{fig:qq} shows the Mahalanobis QQ plot comparing the empirical distribution of SGD at the last iterate with the theoretical HSGD distribution given by \eqref{eq:HSGD_dist}. Given this evidence of distributional equivalence, we assess the Rényi–DP loss using HSGD. In principle, computing the Rényi divergence requires considering all neighboring dataset pairs. To reduce computation, we instead focus on the pairs that are most privacy-adverse.
In our affine-Gaussian case, the dependence on neighboring pairs enters through the parameters $\bm C_i$ and $\bm c_i$, resulting in a mean gap  
$
\Delta m(s^+) = (\bm C - \bm C')E(s) + (\bm c - \bm c'),
$ 
where the discrepancy arises only from the terms $b\bm a - b'\bm a'$ and $\bm a\bm a^\top - \bm a'\bm a'^\top$.  
We define  
$
g(\theta) = (b\bm a - b'\bm a') \;-\; \bigl(\bm a\bm a^\top - \bm a'\bm a'^\top + \delta \bm I_d \bigr)\bm 1_d,
$ where $\bm 1_d$ denotes the vector of all ones of dimension $d$.
Empirically, we observe that larger $\|g(\theta)\|_2$ correlates with larger privacy loss, as shown in Figure~\ref{fig:trend} for a fixed time point.
Therefore, we select neighboring pairs from the existing dataset with the largest $\|g{(\theta)}\|_2$ and produce the privacy trajectories, which provides computationally efficient approximation to the worst case over the observed data.

\paragraph{Privacy-Utility Tradeoffs}
Figure \ref{fig:tradeoff1} compares single-run noisy SGD with the noisy HSGD (indicated as Volterra curves) population risk trajectories across privacy noise scales. For each 
$\sigma$, the empirical curve  of SGD $\mathcal{P}(\bm x_t)$ closely concentrates around the deterministic path $P_t$ given by the Volterra integral equations, and all trajectories converge after some iterations. As $\sigma$ increases, the risk is higher throughout the transient and stabilizes at a higher plateau than for smaller $\sigma$, reflecting the utility cost of added privacy noise.
Because this is one-pass SGD, iteration $k$ equals sample size $k$, meaning the risk value at $k$ is the last-iterate population risk after $k$ samples.

\begin{figure}[h]
  \centering
  \begin{subfigure}[b]{0.35\columnwidth}
    \centering
    \includegraphics[width=\linewidth]{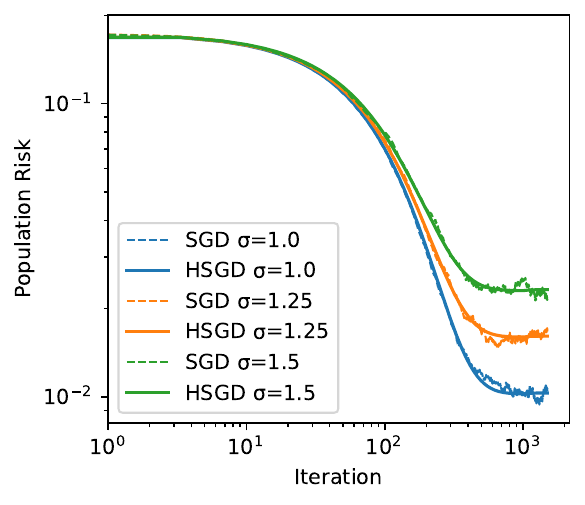}
    \caption{Population risk}
    \label{fig:tradeoff1}
  \end{subfigure}%
  \begin{subfigure}[b]{0.35\columnwidth}
    \centering
    \includegraphics[width=\linewidth]{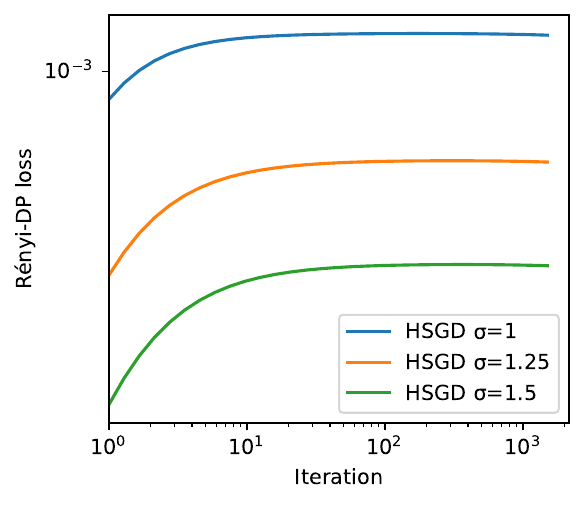}
    \caption{RDP loss}
    \label{fig:tradeoff2}
  \end{subfigure}
  \caption{Statistical risk and privacy loss of noisy SGD via noisy HSGD.}
  \label{fig:tradeoff}
\end{figure}

Figure~\ref{fig:tradeoff2} shows the Rényi–DP loss of HSGD, which approximates that of randomly shuffled SGD, across iterations for a fixed sample size $n=1500$.
The loss rises initially and then plateaus, consistent with last-iterate privacy behavior of other noisy SGD variants studied in previous work. 
This trend is the opposite of the statistical risk in Figure~\ref{fig:tradeoff1}, which decreases and then stabilizes. Together, the two curves illustrate a clear utility–privacy trade-off: larger $\sigma$ yields uniformly smaller privacy loss in both the transient and steady-state regimes with the cost of high statistical risk.

To compare with existing privacy baselines, we also run noisy SGD with gradient clipping using thresholds $c\in\{3,5,7\}$. 
Figure~\ref{fig:comp1} shows population risk trajectories for noisy SGD with $\sigma=1.5$. 
Figure~\ref{fig:comp2} reports the last-iterate privacy-loss trajectories under the same setting as Figure~\ref{fig:comp1}, using theoretical results from \citet{Ye2022DifferentiallyPL} for the clipped version and HSGD for the unclipped. 
The unclipped SGD run attains the lowest statistical risk. 
With similar risk, the privacy bounds of \citep{Ye2022DifferentiallyPL} sit above our HSGD approximation, illustrating a tighter utility–privacy trade-off using the HSGD surrogate. 
Smaller clipping thresholds reduce privacy loss by lowering sensitivity, but they may also increase risk, particularly when set well below the effective sensitivity level.
On the contrary, large clipping thresholds have little effect on statistical risk but will increase privacy loss.
In practice, selecting a suitable clipping threshold is difficult. Our being able to precisely analyze privacy loss for the unclipped variant is therefore valuable for a tighter privacy-utility tradeoff.

\begin{figure}[t]
  \centering
  \begin{subfigure}[b]{0.35\columnwidth}
    \centering
    \includegraphics[width=\linewidth]{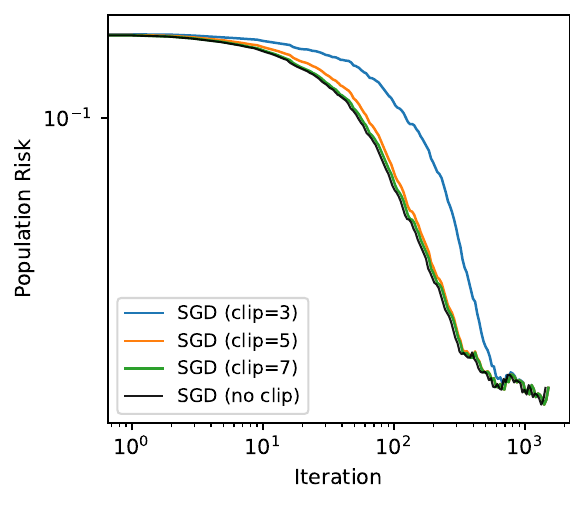}
    \caption{Risk Comparison}
    \label{fig:comp1}
  \end{subfigure}%
  \begin{subfigure}[b]{0.35\columnwidth}
    \centering
    \includegraphics[width=\linewidth]{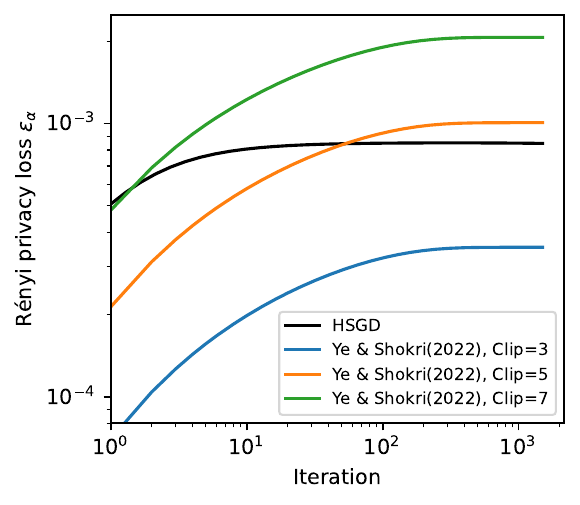}
    \caption{Loss Comparison}
    \label{fig:comp2}
  \end{subfigure}
  \caption{Statistical risk and privacy loss across different clipping thresholds ($\sigma=1.5$).}
  \label{fig:comp}
\end{figure}



\section{Discussion}
We introduce a diffusion approximation, \emph{noisy HSGD}, for one-pass noisy SGD without explicitly knowing gradient sensitivity and show that it is an analytically tractable surrogate for both statistical risk and privacy. In our least square loss setting, noisy HSGD yields closed-form trajectories for the risk and Rényi–DP loss. Theoretically, the statistical risk of noisy HSGD matches that of noisy SGD; empirically, the laws of the two processes are very close, and the observed approximation gap is negligible in the high-dimensional regimes we study. These results support using noisy HSGD trajectories as a mathematically convenient proxy for analyzing noisy SGD.

There are a few future directions to consider. A natural next step is to theoretically quantify the gap between the privacy loss of SGD and that of HSGD, e.g., by proving that the laws of two processes are close. 
Extending from one pass to multi-pass training is more straightforward on the risk side but subtler for privacy analysis: differentiated records are reused across epochs, which corresponds in HSGD to multiple surgical maps and induces inter-time correlations.  
Another natural direction is to extend the quadratic loss we consider to more general loss functions. Smooth and (strongly) convex losses should admit similar diffusion approximations, though the same closed-form Volterra representation may no longer be available, while non-smooth or non-convex settings may require smoothing techniques and a focus on stationary behavior rather than transient risk trajectories.

\bibliographystyle{agsm}
\bibliography{reference} 

@book {Vershynin,
    AUTHOR = {Vershynin, Roman},
     TITLE = {High-dimensional probability},
    SERIES = {Cambridge Series in Statistical and Probabilistic Mathematics},
    VOLUME = {47},
 PUBLISHER = {Cambridge University Press, Cambridge},
      YEAR = {2018},
     PAGES = {xiv+284},
      ISBN = {978-1-108-41519-4},
   MRCLASS = {60-01 (60B05 60B20 60E15 60Fxx 62H25)},
  MRNUMBER = {3837109},
MRREVIEWER = {Sasha\ Sodin},
}

@article{KIM2012regression,
title = {Regression analysis under incomplete linkage},
journal = {Computational Statistics \& Data Analysis},
volume = {56},
number = {9},
pages = {2756-2770},
year = {2012},
author = {Gunky Kim and Raymond Chambers}
}

@inproceedings{dwork2006calibrating,
  title={Calibrating noise to sensitivity in private data analysis},
  author={Dwork, Cynthia and McSherry, Frank and Nissim, Kobbi and Smith, Adam},
  booktitle={Theory of cryptography conference},
  pages={265--284},
  year={2006},
  organization={Springer}
}

@article{dwork2014algorithmic,
  title={The algorithmic foundations of differential privacy},
  author={Dwork, Cynthia and Roth, Aaron},
  journal={Foundations and Trends{\textregistered} in Theoretical Computer Science},
  volume={9},
  number={3--4},
  pages={211--407},
  year={2014},
  publisher={Now Publishers, Inc.}
}

@inproceedings{Bassily2014,
author = {Bassily, Raef and Smith, Adam and Thakurta, Abhradeep},
title = {Private Empirical Risk Minimization: Efficient Algorithms and Tight Error Bounds},
year = {2014},
isbn = {9781479965175},
publisher = {IEEE Computer Society},
address = {USA},
booktitle = {Proceedings of the 2014 IEEE 55th Annual Symposium on Foundations of Computer Science},
pages = {464–473},
numpages = {10},
series = {FOCS '14}
}

@ARTICLE{Paquette2022HSGD,
       author = {{Paquette}, Courtney and {Paquette}, Elliot and {Adlam}, Ben and {Pennington}, Jeffrey},
        title = "{Homogenization of SGD in high-dimensions: Exact dynamics and generalization properties}",
      journal = {arXiv e-prints},
         year = 2022,
        month = may,
          eid = {arXiv:2205.07069},
        pages = {64pp},
archivePrefix = {arXiv},
       eprint = {2205.07069},
 primaryClass = {math.ST},
}

@inproceedings{Bassily2019,
  title={Private Stochastic Convex Optimization with Optimal Rates},
  author={Raef Bassily and Vitaly Feldman and Kunal Talwar and Abhradeep Thakurta},
  booktitle={Neural Information Processing Systems},
  year={2019}
}

@inproceedings{Bassily2020,
author = {Bassily, Raef and Feldman, Vitaly and Guzm\'{a}n, Crist\'{o}bal and Talwar, Kunal},
title = {Stability of stochastic gradient descent on nonsmooth convex losses},
year = {2020},
isbn = {9781713829546},
publisher = {Curran Associates Inc.},
booktitle = {Proceedings of the 34th International Conference on Neural Information Processing Systems},
articleno = {368},
numpages = {11},
location = {Vancouver, BC, Canada},
series = {NIPS'20}
}

@inproceedings{Feldman2020,
author = {Feldman, Vitaly and Koren, Tomer and Talwar, Kunal},
title = {Private stochastic convex optimization: optimal rates in linear time},
year = {2020},
isbn = {9781450369794},
publisher = {Association for Computing Machinery},
booktitle = {Proceedings of the 52nd Annual ACM SIGACT Symposium on Theory of Computing},
pages = {439–449},
numpages = {11},
location = {Chicago, IL, USA},
series = {STOC 2020}
}

@article{chourasia2021differential,
  title={Differential privacy dynamics of langevin diffusion and noisy gradient descent},
  author={Chourasia, Rishav and Ye, Jiayuan and Shokri, Reza},
  journal={Advances in Neural Information Processing Systems},
  volume={34},
  pages={14771--14781},
  year={2021}
}

@article{Ryffel2022DifferentialPG,
  title={Differential Privacy Guarantees for Stochastic Gradient Langevin Dynamics},
  author={Th{\'e}o Ryffel and Francis R. Bach and David Pointcheval},
  journal={arXiv preprint},
  year={2022},
  volume={arXiv:2201.11980}
}

@inproceedings{Paquette2022Implicit,
       author = {{Paquette}, Courtney and {Paquette}, Elliot and {Adlam}, Ben and {Pennington}, Jeffrey},
        title = "{Implicit Regularization or Implicit Conditioning? Exact Risk Trajectories of SGD in High Dimensions}",
       	booktitle = {Advances in Neural Information Processing Systems},
	editor = {S. Koyejo and S. Mohamed and A. Agarwal and D. Belgrave and K. Cho and A. Oh},
      volume = {35},
         year = 2022,
        month = jun,
          eid = {arXiv:2206.07252},
        pages = {33pp},
archivePrefix = {arXiv},
       eprint = {2206.07252},
 primaryClass = {stat.ML},
}

@article{Altschuler2022PrivacyON,
  title={Privacy of Noisy Stochastic Gradient Descent: More Iterations without More Privacy Loss},
  author={Jason M. Altschuler and Kunal Talwar},
  journal={arXiv preprint},
  year={2022},
  volume={arXiv:2205.13710},
}

@article{Ye2022DifferentiallyPL,
  title={Differentially Private Learning Needs Hidden State (Or Much Faster Convergence)},
  author={Jiayuan Ye and R. Shokri},
  journal={arXiv preprint},
  year={2022},
  volume={arXiv:2203.05363},
}

@article{Elizabeth2024,
author = {Collins–Woodfin, Elizabeth and Paquette, Elliot},
year = {2024},
month = {01},
pages = {},
title = {High-dimensional limit of one-pass SGD on least squares},
volume = {29},
journal = {Electronic Communications in Probability},
doi = {10.1214/23-ECP571}
}

@article{Mironov2017RnyiDP,
  title={R{\'e}nyi Differential Privacy},
  author={Ilya Mironov},
  journal={2017 IEEE 30th Computer Security Foundations Symposium (CSF)},
  year={2017},
  pages={263-275},

}

@article{asoodeh2023privacy,
  title={Privacy loss of noisy stochastic gradient descent might converge even for non-convex losses},
  author={Asoodeh, Shahab and Diaz, Mario},
  journal={arXiv preprint arXiv:2305.09903},
  year={2023}
}

@inproceedings{Song2021EvadingTC,
  title={Evading the Curse of Dimensionality in Unconstrained Private GLMs},
  author={Shuang Song and Thomas Steinke and Om Thakkar and Abhradeep Thakurta},
  booktitle={International Conference on Artificial Intelligence and Statistics},
  year={2021},

}

@article{Bombari2025BetterRates,
  title        = {Better Rates for Private Linear Regression in the Proportional Regime via Aggressive Clipping},
  author       = {Simone Bombari and Inbar Seroussi and Marco Mondelli},
  journal      = {arXiv preprint arXiv:2505.16329},
  year         = {2025},
  eprint       = {2505.16329},
  archivePrefix= {arXiv},
  primaryClass = {stat.ML},
  doi          = {10.48550/arXiv.2505.16329}
}

@article{Wu2021LastIR,
  title={Last Iterate Risk Bounds of SGD with Decaying Stepsize for Overparameterized Linear Regression},
  author={Jingfeng Wu and Difan Zou and Vladimir Braverman and Quanquan Gu and Sham M. Kakade},
  journal={ArXiv},
  year={2021},
  volume={abs/2110.06198},

}

@article{Gil2013RnyiDM,
  title={R{\'e}nyi divergence measures for commonly used univariate continuous distributions},
  author={Manuel Gil and Fady Alajaji and Tam{\'a}s Linder},
  journal={Inf. Sci.},
  year={2013},
  volume={249},
  pages={124-131},
  url={https://api.semanticscholar.org/CorpusID:15181549}
}

\clearpage
\appendix
\section*{Supplementary Material: Technical Proofs}
\addcontentsline{toc}{section}{Supplementary Material: Technical Proofs}

\setcounter{section}{0}
\renewcommand{\thesection}{\Alph{section}}

\section{Proofs for Section 3}

We employ the martingale method in diffusion approximations to prove Theorem~\ref{thm:comparison}. Essentially, we decompose the risk into predictable and martingale components using the Doob decomposition technique. 
We establish the high-dimensional equivalence of noisy SGD and noisy HSGD in terms of quadratic risk by demonstrating that the predictable parts of the risk under noisy SGD and noisy HSGD are sufficiently close, while the martingale error terms become negligible as $d$ grows large.

Note that $\nabla f_k(\bm x) = \bm a_k(\bm a_k^\top \bm x - b_k) + \delta \bm x$, then the $k$-th iterate can be written recursively as
\begin{equation}
\label{eq:recurrence1}
\begin{aligned}
    \bm x_{k+1} = \bm x_{k} - \eta_k (f_{k}(\bm x_{k}) + \sigma \bm Z_k)
    = \bm x_{k} - \eta_k \bm A^\top \bm e_{k} \bm e_{k}^\top (\bm A\bm x_{k} - \bm b) - \eta_k \delta \bm x_{k} - \eta_k \sigma \bm Z_k.
\end{aligned}
\end{equation}
With \(b_k=\bm a_k^\top \tilde{\bm x}+\xi_k\) as assumed, we have
\begin{equation}
\label{eq:recurrence}
    \bm x_{k+1} - \tilde{\bm x} = ((1-\eta_k\delta)\bm I_d - \eta_k \bm a_k \bm a_k^\top)(\bm x_{k} - \tilde{\bm x}) - \eta_k \delta \tilde{\bm x} + \eta_k\bm a_k \xi_k - \eta_k\sigma\bm Z_k,
\end{equation}
where $\tilde{\bm x}$ denotes the ground truth. Let
\[
\bm v_k := \bm x_k - \tilde{\bm x} \quad \text{and} \quad
\bm V_t: = \bm X_t - \tilde{\bm x},
\]
and $(\mathcal{\bm F}_t: t\geq 0)$ be the filtration generated by $(\bm v_t: t\geq 0)$ and $(\bm V_{\lfloor t/d \rfloor}: t\geq 0)$. Then, $\bm v_k$ is $\mathcal{\bm F}_k$-measurable.
Let $q(\cdot)$ be any quadratic loss and $\gamma_k = \eta_kd$ the rescaled learning rate for noisy HSGD. 
By (\ref{eq:recurrence}),
\begin{equation}
\label{eq:increment}
\begin{aligned}
    q(\bm v_k) - q(\bm v_{k-1}) & = 
    -\gamma_k(\nabla q(\bm v_{k-1}))^\top \left(\bm u_{k-1} + \bm \Delta_k + \frac{\sigma}{d}\bm Z_k\right)  \\
    & + \frac{\gamma_k^2}{2}\left(\bm u_{k-1} + \bm \Delta_k + \frac{\sigma}{d}\bm Z_k\right)^\top (\nabla^2q)\left(\bm u_{k-1} + \bm \Delta_k + \frac{\sigma}{d}\bm Z_k\right),
\end{aligned}
\end{equation}
where 
\begin{equation*}
    \bm \Delta_k = \bm m_k(\bm m_k^\top \bm v_{k-1} - \xi_k/\sqrt{d}),
    \quad 
    \bm m_k = \bm a_k/\sqrt{d},
    \quad
    \bm u_{k-1} = \frac{\delta}{d}(\bm v_{k-1} + \tilde{\bm x}).
\end{equation*}
In the following, we decompose $q(\bm v_k) - q(\bm v_{k-1})$ into a predictable part and martingale increments.
In particular, the predictable part is broken down into a leading-order term and a higher-order term $\Delta\mathcal{E}_k$.

By expanding \eqref{eq:increment}
and grouping: (i) predictable terms, (ii) martingale increments from $\bm r_k$, (iii) martingale increments from $\bm Z_k$, and (iv) higher order predictable errors, we have the following result:
\begin{lemma}[Doob decomposition]
\label{lem:doob-dp-nosplit}
Let $q:\mathbb{R}^d\to\mathbb{R}$ be quadratic with constant Hessian $\bm H=\nabla^2 q$.
Define
\[
\bm v_k=\bm x_k-\tilde{\bm x},\qquad
\bm u_{k-1}=\tfrac{\delta}{d} (\bm v_{k-1}+\tilde{\bm x}),\qquad
\bm\Delta_k=\bm m_k(\bm m_k^\top \bm v_{k-1}-\frac{\xi_k}{\sqrt d}),
\]
with $\bm m_k=\bm a_k/\sqrt d$, and let $\bm Z_k\sim\mathcal N(\bm 0,\bm I_d)$ be independent DP noise (independent of data), scaled by $\sigma/d$.
Let $\mathcal F_k$ be the natural filtration. Then for each $k\ge1$,
\[
q(\bm v_k)-q(\bm v_{k-1})
= \Delta q_k^{\mathrm{pred}}
+ \Delta \mathcal M_k^{\mathrm{lin}}
+ \Delta \mathcal M_k^{\mathrm{quad}}
+ \Delta \mathcal M_k^{\mathrm{lin,DP}}
+ \Delta \mathcal M_k^{\mathrm{quad,DP}}
+ \Delta \mathcal E_k^{\mathrm{pred}},
\]
where:
\begin{align*}
\Delta q_k^{\mathrm{pred, SGD}}
&= -\frac{\gamma_k}{d} (\nabla q(\bm v_{k-1}))^\top\!\Big[(\delta \bm I+\bm\Sigma)\bm v_{k-1}+\delta \tilde{\bm x}\Big]
\\
&\quad
+ \frac{\gamma_k^2}{2d^2} \tr(\bm\Sigma \bm H) 
\Big(\bm v_{k-1}^\top \bm\Sigma \bm v_{k-1} + \mathbb E[\xi_k^2]\Big)
+ \frac{\gamma_k^2\sigma^2}{2d^2} \tr(\bm H),
\\[0.75ex]
\Delta \mathcal M_k^{\mathrm{lin}}
&= -\gamma_k (\nabla q(\bm v_{k-1}))^\top\!\Big(\bm\Delta_k-\mathbb E[\bm\Delta_k\mid\mathcal F_{k-1}]\Big)
+ \gamma_k^2 \bm u_{k-1}^\top \bm H \Big(\bm\Delta_k-\mathbb E[\bm\Delta_k\mid\mathcal F_{k-1}]\Big),
\\[0.75ex]
\Delta \mathcal M_k^{\mathrm{quad}}
&= \frac{\gamma_k^2}{2}\Big(
\bm\Delta_k^\top \bm H \bm\Delta_k
- \mathbb E[\bm\Delta_k^\top \bm H \bm\Delta_k\mid\mathcal F_{k-1}]
\Big),
\\[0.75ex]
\Delta \mathcal M_k^{\mathrm{lin,DP}}
&= -\gamma_k (\nabla q(\bm v_{k-1}))^\top \frac{\sigma}{d} \bm Z_k
+ \gamma_k^2 (\bm u_{k-1} + \bm\Delta_k)^\top \bm H \frac{\sigma}{d} \bm Z_k,
\\[0.75ex]
\Delta \mathcal M_k^{\mathrm{quad,DP}}
&= \frac{\gamma_k^2}{2} \frac{\sigma^2}{d^2} 
\Big(\bm Z_k^\top \bm H \bm Z_k - \tr(\bm H)\Big),
\\[0.75ex]
\Delta \mathcal E_k^{\mathrm{quad}}
&= \frac{\gamma_k^2}{2}\Big[
\bm u_{k-1}^\top \bm H \bm u_{k-1}
+ 2 \bm u_{k-1}^\top \bm H \mathbb E[\bm\Delta_k\mid\mathcal F_{k-1}]
+ \mathbb E[\bm\Delta_k\mid\mathcal F_{k-1}]^\top \bm H \mathbb E[\bm\Delta_k\mid\mathcal F_{k-1}]
\Big]
\\
&\quad
+ \frac{\gamma_k^2}{2}\Big(
\mathbb E[\bm\Delta_k^\top \bm H \bm\Delta_k\mid\mathcal F_{k-1}]
- \tfrac{1}{d^2}\tr(\bm\Sigma \bm H) 
\big(\bm v_{k-1}^\top \bm\Sigma \bm v_{k-1} + \mathbb E[\xi_k^2]\big)
\Big).
\end{align*}
Moreover, each of
\(\Delta \mathcal M_k^{\mathrm{lin}}, \Delta \mathcal M_k^{\mathrm{quad}},
\Delta \mathcal M_k^{\mathrm{lin,DP}}, \Delta \mathcal M_k^{\mathrm{quad,DP}}\)
is a martingale difference w.r.t.\ $\{\mathcal F_k\}$, i.e.
$\mathbb E[\cdot\mid\mathcal F_{k-1}]=0$.
\end{lemma}

The noisy HSGD $\bm X_t$ is constructed in a way that the its predictable part is structurally the same as the counterpart of SGD.
Recall $\bm V_t := \bm X_t - \bm {\tilde x}$, then we have
\begin{equation}
    d\bm V_t = -\gamma(t) \nabla \mathcal{R}(\bm V_t +  \tilde{\bm x}) \mathrm{d}t + \gamma(t)\sqrt{\frac{1}{d}\left( 2
\mathcal{P}(\bm V_t + \tilde{\bm x}) \bm\Sigma + \sigma^2 \bm I_d \right)}\mathrm{d}B_t.
\end{equation}
Using Itô's Lemma, it follows that 
\begin{equation}
\begin{aligned}
     dq(\bm V_t)
     & = - \gamma(t) (\nabla q (\bm V_t))^\top\nabla\mathcal{R}(\bm V_t + \tilde{\bm x})\mathrm{d}t\\
     & + \frac{\gamma^2(t)}{d} \mathcal{P}(\bm V_t+ \tilde{\bm x}) \operatorname{tr}\left(\bm \Sigma(\nabla^2 q)\right) \mathrm{d}t
     +
      \frac{\gamma^2(t)\sigma^2}{2d} \operatorname{tr}\left(\nabla^2 q\right) \mathrm{d}t
      \\
     & + \gamma(t)(\nabla q (\bm V_t))^\top\sqrt{\frac{1}{d}\left( 2
    \mathcal{P}(\bm V_t + \tilde{\bm x}) \bm\Sigma + \sigma^2 \bm I_d \right)}\mathrm{d}B_t \\
    & =: \mathrm{d} q_t^{pred,HSGD} + \mathrm{d} \mathcal{M}_t^{HSGD}
\end{aligned}
\end{equation}
where $\mathrm{d} q_t^\mathrm{pred, HSGD}$ consists first three terms and $\mathrm{d} \mathcal{M}_t^{\mathrm{HSGD}}$ denotes the last term.
To compare $\mathrm{d} q_t^{\mathrm{pred, HSGD}}$ and $\Delta q_k^{\mathrm{pred, SGD}}$,  
tote that the population risk $\mathcal{P}$ can be rewritten as
\begin{equation*}
\label{eq:P_gen}
    \mathcal{P}(\bm x ) = \frac{1}{2}[(\bm x - \tilde{\bm x})^\top \bm \Sigma (\bm x - \tilde{\bm x}) + \mathbb{E}\xi^2],
\end{equation*}
where $\xi$ is the regression modeling noise.
Also note that the gradient of the regularized risk
$\mathcal{R}$ is given by
\begin{equation*}
    \nabla \mathcal{R}(\bm x ) = \bm\Sigma (\bm x - \tilde{\bm x})+\delta \bm x= (\bm\Sigma+\delta \bm I_d)\bm (\bm x - \tilde{\bm x})+\delta \tilde{\bm x}.
\end{equation*}
Replace the expressions of $\mathcal{P}$ and $\nabla\mathcal{R}$ in the predictable term of SGD increment,
$\Delta q_k^{\mathrm{pred, SGD}}$ can be rewritten as
\begin{equation}
\begin{aligned}
    \Delta q_k^{\mathrm{pred, SGD}} & = -\frac{\gamma_k}{d} (\nabla q(\bm v_{k-1}))^\top \nabla \mathcal{R}(\bm v_{k-1}+ \tilde{\bm x})
    +\frac{\gamma_k^2}{d^2}\mathcal{P}(\bm v_{k-1}+ \tilde{\bm x}) \operatorname{tr}\left(\bm \Sigma(\nabla^2 q)\right) +
    \frac{\gamma_k^2\sigma^2}{2d^2}\operatorname{tr}\left(\nabla^2 q\right).
\end{aligned}
\end{equation}
which corresponds precisely to the predicable part of the continuous-time process HSGD with the change of time variable, i.e., from $\bm x_{\lfloor td \rfloor}$ to $\bm X_t$.

Compared to the non-private setting, the predictable part contains an additional state-independent term, namely
\[
\frac{\gamma_k^2\sigma^2}{2d^2} \operatorname{tr}\!\left(\nabla^2 q\right)
\quad\text{in noisy SGD, and}\quad
\frac{\gamma^2(t)\sigma^2}{2d} \operatorname{tr}\!\left(\nabla^2 q\right) \mathrm{d}t
\quad\text{in noisy HSGD}.
\]
Since this correction appears in both processes with the same form, verifying alignment of the predictable parts reduces to the non-private case.  
As it has already been shown that
\[
\bigl| q(\bm v_k) - q(\bm V_{k/d}) \bigr| \to 0 \qquad \text{as } d\to\infty,
\]
it remains only to control the error terms induced by the injected noise. These correspond to the martingale components
\(\mathcal M_k^{\mathrm{lin,DP}}, \mathcal M_k^{\mathrm{quad,DP}}, \mathcal M_t^{\mathrm{HSGD}}\),
which arise as the sums or integrals of their increments
\(\Delta \mathcal M_k^{\mathrm{lin,DP}},
  \Delta \mathcal M_k^{\mathrm{quad,DP}},
  \mathrm{d}\mathcal M_t^{\mathrm{HSGD}}\),
as defined above.

To control the martingale terms, we introduce the stopping time
\begin{equation}
    \tau := \inf\{k:\|\bm v_k\|>d^\alpha\}\cup\{td:\|\bm V_t\|>d^\alpha\},
\end{equation}
and define the corresponding stopped processes by
\begin{equation}
    \bm v_k^{\tau} := \bm v_{k\wedge \tau},
    \qquad
    \bm V_t^{\tau} := \bm V_{t\wedge (\tau/d)}.
\end{equation}

In the non-private setting (i.e., with $\sigma=0$), \citet{Elizabeth2024} shows that the stopped martingale and error terms are controlled for any quadratic $q$ with $\|q\|_{C^2}\le 1$. Formally:
\begin{lemma}{\citep[Lemma~2.3 and Lemma~3.2]{Elizabeth2024}}
    For any quadratic \(q\) with \(\|q\|_{C^2}\le 1\), the stopped martingale and error terms
\(\mathcal M_{k}^{\mathrm{lin},\tau},
  \mathcal M_{k}^{\mathrm{quad},\tau},
  \mathcal E_{k}^{\mathrm{quad},\tau}\)
satisfy, uniformly in \(q\), the following bounds with overwhelming probability for \(n\le d\ln d\):
\begin{align*}
\text{(i)}\quad & \sup_{1\le k\le n} \bigl|\mathcal M_{k}^{\mathrm{lin},\tau}\bigr|
    \le d^{-1/2+5c_*}, \\[0.5ex]
\text{(ii)}\quad & \sup_{1\le k\le n} \bigl|\mathcal M_{k}^{\mathrm{quad},\tau}\bigr|
    \le d^{-1/2+9c_*}, \\[0.5ex]
\text{(iii)}\quad & \sup_{1\le k\le n} \bigl|\mathcal E_{k}^{\mathrm{quad},\tau}\bigr|
    \le d^{-1+9c_*}, \\
  \text{(iv)}\quad & \sup_{1\le t\le n/d} \bigl|\mathcal M_{t}^{\mathrm{HSGD},\tau}|\lesssim d^{-1/2+3c_*}.  
\end{align*}
\end{lemma}

Upon which, the equivalence between non-private SGD and non-private HSGD is shown:
\begin{theorem}{\cite[Theorem~1.5]{Elizabeth2024}}
For any quadratic loss function $q$ with bounded $\|\cdot\|_{C^2}$-norm, there exists a constant $c_0$ such that, for any $n\le d\ln d / c_0$, w.o.p.\ we have
\begin{equation}
\begin{aligned}
\sup_{0 \le k \le n} \big|q(\bm x_k) - q(\bm X_{k/d})\big|
  \le \|q\|_{C^2} e^{c_0 n/(8d)} \times 
    d^{-\tfrac12 + 9c_*}.
\end{aligned}
\end{equation}
\end{theorem}
Therefore, given the similarity to the non-private setting, we only detail how the additional martingales due to the injected noise, $\mathcal M_{k}^{\mathrm{lin,DP},\tau}$ and $\mathcal M_{k}^{\mathrm{quad,DP},\tau}$, as well as the private version of HSGD martingale $\mathcal M_{t}^{\mathrm{HSGD},\tau}$, are controlled.
The remaining arguments that establish the high-dimensional equivalence mirror those of the non-private case and are not repeated here.

\begin{lemma}[Bounds for DP--noise martingales]
\label{lem:dp-martingale-bounds} 
For any quadratic $q$ with $\|q\|_{C^2}\le 1$,after imposing the stopping time $\tau$, the
resulting martingales due to injected noise satisfy  following bounds with overwhelming probability provided $n\le d\ln d$,
\begin{align*}
\text{(i)}\quad & \sup_{1\le k\le n} \bigl|\mathcal M_{k}^{\mathrm{lin,DP},\tau}\bigr|
    \le \sigma d^{-1/2+2c_*}, \\[0.5ex]
\text{(ii)}\quad & \sup_{1\le k\le n} \bigl|\mathcal M_{k}^{\mathrm{quad,DP},\tau}\bigr|
    \le \sigma^2 d^{-1/2}, \\[0.5ex]
\text{(iii)}\quad & \sup_{0\le t\le n/d} \bigl|\mathcal M_{t}^{\mathrm{HSGD},\tau}\bigr|
    \lesssim d^{-1/2+3c_*} + \sigma d^{-1/2 + 2 c_*}.
\end{align*}
\end{lemma}

To establish these bounds, we make use of the following concentration lemma. 
This result is standard in \citep{Vershynin} which treats the 
non-martingale case. The martingale version below is provided in \citet{Elizabeth2024} (Lemma~3.1).
\begin{lemma}[Martingale Bernstein inequality]
\label{lem:bernstein}
Let $(M_n)_{n=0}^N$ be a martingale on the filtered probability space 
$(\Omega, (\mathcal F_n)_{n=0}^N, \mathbb P)$. 
For $p \ge 1$, define
\begin{equation}
\label{eq:sigmanp}
\sigma_{n,p} 
:= \left\|  \inf\Big\{ t \ge 0 : 
\mathbb E\!\left[\exp\!\big(|M_n - M_{n-1}|^p / t^p\big) \middle| \mathcal F_{n-1}\right] 
\le 2 \Big\} \right\|_{L^\infty(\mathbb P)}.
\end{equation}
Then there exists an absolute constant $C>0$ such that, for all $t>0$,
\begin{equation}
\label{eq:bernstein}
\mathbb P\!\left( \sup_{1\le n\le N} |M_n - M_0|  \ge  t \right)
\le
2 \exp\!\left( 
-  \min\left\{ \frac{t}{C \max_{1\le n\le N}\sigma_{n,1}},\ 
\frac{t^2}{C \sum_{n=1}^N \sigma_{n,1}^2} \right\}
\right).
\end{equation}
\end{lemma}

\begin{proof}[Proof of Lemma 1.2 (i)]
Let $\zeta_k := \frac{\sigma}{d} Z_k$. Then
\[
\Delta M^{\mathrm{lin,DP},\tau}_k
= -\gamma_k \nabla q(v_{k-1})^\top \zeta_k
+ \gamma_k^2 u_{k-1}^\top H \zeta_k
+ \gamma_k^2 \Delta_k^\top H \zeta_k.
\]
By the definition of $\psi_1$--Orlicz norms and the stopping-time bounds,
\begin{align*}
\big\|  \nabla q(v_{k-1})^\top \zeta_k \big\|_{\psi_1}
&\lesssim \frac{\sigma}{d}\|\nabla q(v_{k-1})\|
\lesssim \sigma d^{-1+c_*}, \\[0.5ex]
\big\|  u_{k-1}^\top H \zeta_k \big\|_{\psi_1}
&\lesssim \frac{\sigma}{d}\|u_{k-1}\|
\lesssim \sigma\ d^{-2+c_*}, \\[0.5ex]
\big\|  \Delta_k^\top H \zeta_k \big\|_{\psi_1}
&\lesssim 2\frac{\sigma}{d}\|\Delta_k\|
\lesssim \sigma d^{-2+c_*}.
\end{align*}
Therefore,
$
\sigma_{k,1} \lesssim \sigma d^{-1+c_*}.
$
For $n \le d\ln d$, the cumulative variance proxy satisfies
\[
\sum_{k=1}^n \sigma_{k,1}^2
\lesssim n\sigma^2  d^{-2+2c_*}
\le \sigma^2  d^{-1+2c_*} \ln d.
\]
Applying Lemma~\ref{lem:bernstein} with $p=1$, for any $t>0$,
\[
\Pr\!\left(\sup_{1\le k\le n} |M^{\mathrm{lin,DP},\tau}_k| \ge t\right)
\le
2\exp\!\left(
-\min\left\{ \frac{t}{C \max_{k\le n} \sigma_{k,1}},\ 
\frac{t^2}{C\sum_{k=1}^n \sigma_{k,1}^2}
\right\}
\right).
\]
Taking
$
t = \sigma\ d^{-1/2+2c_*},
$
we obtain
\[
\frac{t}{\max_k \sigma_{k,1}} \gtrsim d^{1/2+2c_*},
\qquad
\frac{t^2}{\sum_k\sigma_{k,1}^2} \gtrsim d^{2c_*} (\ln d)^{-1/2}.
\]
Therefore, with overwhelming probability,
\[
\sup_{1\le k\le n} |M^{\mathrm{lin,DP},\tau}_k|
\le\sigma d^{-1/2+2c_*}.
\]

\end{proof}

\begin{proof}[Proof of Lemma 1.2 (ii)]
Since $Z\sim\mathcal N(0,I_d)$, by the Hanson--Wright inequality we have
\[
\|Z^\top H Z - \mathrm{tr}(H)\|_{\psi_1} \lesssim \|H\|_F.
\]
Given $\|q\|_{C^2}\le 1$, it follows that $\|H\|_F\le \sqrt{d}$.
Therefore,
\[
\sigma_{k,1}
:= \left\|\inf\{t>0:\ \mathbb E[e^{|\Delta M^{\mathrm{quad,DP},\tau}_k|/t}\mid\mathcal F_{k-1}]\le2\}\right\|_{L^\infty}
  \lesssim   \sigma^2  d^{-3/2}.
\]
For $n\le d\ln d$,
\[
\sum_{k=1}^n \sigma_{k,1}^2
  \lesssim   n  \sigma^4  d^{-3}
  \le   \sigma^4  d^{-2}  \ln d,
\qquad
\max_{k\le n}\sigma_{k,1}  \lesssim  \sigma^2  d^{-3/2}.
\]
Applying Bernstein’s inequality with $p=1$, for any $t>0$,
\[
\Pr\!\left(\sup_{1\le k\le n}\Big|{\textstyle\sum_{j=1}^k}\Delta \mathcal M^{\mathrm{quad,DP},\tau}_j\Big|\ge t\right)
\le
2\exp\!\left(
-\min\!\left\{\frac{t}{\max_{k\le n}\sigma_{k,1}},  
\frac{t^2}{\sum_{k=1}^n \sigma_{k,1}^2}
\right\}\right).
\]
Taking
$t   =   \sigma^2  d^{-1/2},$
we obtain
\[
\frac{t}{\max_k\sigma_{k,1}}   \gtrsim   d,
\qquad
\frac{t^2}{\sum_k\sigma_{k,1}^2}   \gtrsim   \frac{d}{\ln d}.
\]
That is,
$\sup_{1\le k\le n}\big|\mathcal M^{\mathrm{quad,DP},\tau}_k\big|\ \le\ \sigma^2  d^{-1/2}$ with overwhelming probability.
\end{proof}

\begin{proof}[Proof of Lemma 1.2 (iii)]
The quadratic variation is
\begin{equation*}\label{eq:M_gen_qv}
\big[\mathcal{M}^{\mathrm{HSGD},\tau}\big]_t
=\int_0^t \frac{\gamma(s)^2}{d}  
\big(\nabla q(\bm V_s^\tau)\big)^{\!\top}
\Big( 2  \mathcal P(\bm V_s^\tau+\tilde{\bm x})  \bm \Sigma + \sigma^2 \bm I_d \Big)  
\nabla q(\bm V_s^\tau)  ds.
\end{equation*}
From $\|q\|_{C^2}\le 1$ we have $\|\nabla q(x)\|\le \|x\|+1$, hence on $\{\|V_s^\tau\|\le d^{c_*}\}$,
\begin{equation*}\label{eq:gradq_bound}
\|\nabla q(\bm V_s^\tau)\|^2 \lesssim d^{2c_*}.
\end{equation*}
From \eqref{eq:P_gen},
\begin{equation*}\label{eq:P_pointwise_bound}
\mathcal P(\bm V_s^\tau+\tilde{x})
=\tfrac12\Big( (\bm V_s^\tau)^{\!\top}\bm \Sigma \bm V_s^\tau + \mathbb{E}[\xi^2] \Big)
 \le \tfrac12\big(\|\bm \Sigma\|  \|\bm V_s^\tau\|^2 + \mathbb{E}[\xi^2]\big)
\lesssim \big(\|\bm \Sigma\|  d^{2c_*} + 1\big).
\end{equation*}
With $n\le d\ln d$,
\begin{align}
\sup_{0\le t\le n/d}\big[\mathcal{M}^{\mathrm{HSGD},\tau}\big]_t
&\lesssim \frac{1}{d}\int_0^{n/d}
\Big( 2  \mathcal P(\bm V_s^\tau+\tilde{x})  \|\bm \Sigma\| + \sigma^2 \Big)  
\|\nabla q(\bm V_s^\tau)\|^2  ds \nonumber\\
&\lesssim \frac{n}{d^2}  
\Big( 2  \|\bm \Sigma\|  \big(\|\bm \Sigma\|  d^{2c_*}+1\big) + \sigma^2 \Big)    d^{2c_*} \nonumber\\
&\lesssim 
d^{-1+4c_*}\ln d + \sigma^2 d^{-1+2c_*}\ln d
\label{eq:qv_uniform_bound}
\end{align}
By the Gaussian tail bound for continuous martingales of bounded
quadratic variation, 
\begin{equation*}
    \Pr\!\big(\sup_{0\le t\le n/d}|\mathcal{M}^{\mathrm{HSGD},\tau}_t|>t\big)
\le \exp\!\big(-t^2/\big(2\sup_{t\le n/d}[\mathcal{M}^{\mathrm{HSGD},\tau}]_t\big)\big).
\end{equation*}
Combining with \eqref{eq:qv_uniform_bound} and choosing $t=d^{-1/2+3c_*} + \sigma d^{-1/2 + 2 c_*}$,
\label{eq:M_sup_bound}
$$\frac{t^2}{2\sup_{t\le n/d}[\mathcal{M}^{\mathrm{HSGD},\tau}]_t} \gtrsim d^{2c_*}/\ln d.$$
Hence, with overwhelming probability,
\begin{equation*}
    \sup_{0\le t\le n/d}\big|\mathcal{M}^{\mathrm{HSGD},\tau}_t\big|
\lesssim d^{-1/2+3c_*} + \sigma d^{-1/2 + 2 c_*}.
\end{equation*}
\end{proof}

\section{Proofs for Section 4.1} %
\begin{proof}[\textbf{Proof of Theorem~\ref{thm:concen_risk}}]
We work in continuous time $t$. Let $P_t := \mathbb{E}[\mathcal{P}(\bm X_t)]$ and $R_t := \mathbb{E}[\mathcal{R}(\bm X_t)]$.
We prove concentration for $\mathcal{P}(\bm X_t)$; the $\mathcal{R}$ case is analogous.

Let $\bm A:=\bm\Sigma+\delta \bm I_d$ and $\bm Q_t := \exp(\bm A  \Gamma(t))$ where $\Gamma(t) = \int_0^t\gamma(u)du$. By Itô's lemma,
\[
d(\bm Q_t \bm X_t)
= \gamma(t)  \bm Q_t  \bm\Sigma  \tilde{\bm x}  dt
  + \gamma(t)  \bm Q_t \sqrt{\tfrac{1}{d}\big(2\mathcal{P}(\bm X_t)\bm\Sigma+\sigma^2 \bm I_d\big)}  d\bm B_t,
\]
hence, integrating from $0$ to $t$,
\[
\bm X_t
= \bm X^{\mathrm{gf}}_{\Gamma(t)}
 + \bm Q_t^{-1}\!\int_0^t \gamma(s)  \bm Q_s \sqrt{\tfrac{1}{d}\big(2\mathcal{P}(\bm X_s)\bm\Sigma+\sigma^2 \bm I_d\big)}  d\bm B_s,
\]
where $\bm X^{\mathrm{gf}}$ solves the gradient flow $d\bm X_t=-\gamma(t)\nabla \mathcal{R}(\bm X_t)  dt$ with the same initial condition.
Applying $\mathcal{P}$ and using $\nabla\mathcal{P}(\bm x)=\bm\Sigma(\bm x-\tilde{\bm x})$ and Itô’s formula yields
\begin{equation}
\label{eq:risk_p_xt}
\begin{aligned}
\mathcal{P}(\bm X_t)
&= \mathcal{P}(\bm X^{\mathrm{gf}}_{\Gamma(t)})
  + \nabla\mathcal{P}(\bm X^{\mathrm{gf}}_{\Gamma(t)})^\top
    \bm Q_t^{-1}\!\int_0^t \gamma(s)  \bm Q_s \sqrt{\tfrac{1}{d}\big(2\mathcal{P}(\bm X_s)\bm\Sigma+\sigma^2 \bm I_d\big)}  d\bm B_s \\
&\quad + \frac{1}{2}\Big\|
\bm\Sigma^{-1/2}\bm Q_t^{-1}\!\int_0^t \gamma(s)  \bm Q_s \sqrt{\tfrac{1}{d}\big(2\mathcal{P}(\bm X_s)\bm\Sigma+\sigma^2 \bm I_d\big)}  d\bm B_s
\Big\|^2.
\end{aligned}
\end{equation}
Taking expectations gives
\begin{equation}
\label{eq:expectation_risk_p}
P_t
= \mathcal{P}(\bm X^{\mathrm{gf}}_{\Gamma(t)})
 + \frac{1}{d}\!\int_0^t \gamma^2(s)  \mathrm{tr}\!\big(\bm\Sigma^2  \bm Q_t^{-2}\bm Q_s^{2}\big)  P_s  ds
 + \frac{\sigma^2}{2d}\!\int_0^t \gamma^2(s)  \mathrm{tr}\!\big(\bm\Sigma  \bm Q_t^{-2}\bm Q_s^{2}\big)  ds.
\end{equation}
Let $\Delta_t:=\mathcal{P}(\bm X_t)-P_t$. Subtracting \eqref{eq:expectation_risk_p} from \eqref{eq:risk_p_xt} and inserting
$\mathcal P(\bm X_s)=P_s+\Delta_s$ inside the last term gives the decomposition
\[
\Delta_t
= \mathcal{M}^{(1)}_t + \mathcal{M}^{(2)}_t
 + \frac{1}{d}\!\int_0^t \gamma^2(s)  \mathrm{tr}\!\big(\bm\Sigma^2  \bm Q_t^{-2}\bm Q_s^2\big)  \Delta_s  ds,
\]
where
\[
\mathcal{M}^{(1)}_t
:= \nabla\mathcal{P}(\bm X^{\mathrm{gf}}_{\Gamma(t)})^\top
    \bm Q_t^{-1}\!\int_0^t \gamma(s)  \bm Q_s \sqrt{\tfrac{1}{d}\big(2\mathcal{P}(\bm X_s)\bm\Sigma+\sigma^2 \bm I_d\big)}  d\bm B_s,
\]
and
\begin{equation*}
    \begin{aligned}
        \mathcal{M}^{(2)}_t
& := \frac{1}{2}\Big\|
\bm\Sigma^{-1/2}\bm Q_t^{-1}\!\int_0^t \gamma(s)  \bm Q_s \sqrt{\tfrac{1}{d}\big(2\mathcal{P}(\bm X_s)\bm\Sigma+\sigma^2 \bm I_d\big)}  d\bm B_s
\Big\|^2
\\
& \qquad-\frac{1}{2d}\!\int_0^t \gamma^2(s)  \mathrm{tr}\!\big(\bm\Sigma  \bm Q_t^{-2}\bm Q_s^{2}\big(2\mathcal{P}(\bm X_s)\bm\Sigma+\sigma^2\bm I_d\big)\big)  ds.
    \end{aligned}
\end{equation*}
(Note $\|\bm Q_t^{-1}\bm Q_s\|_{\op}\le 1$ for $t\ge s$ since $\bm A\succeq 0$.)

Now fix the high-probability event from Lemma~\ref{lem:boundednessofP} (below) on which
\[
\sup_{s\le T}\mathcal{P}(\bm X_s) \le  C  e^{C T+(\ln d)^{3/4}}+C d^{2c_*}.
\]
On this event, Lemma~\ref{lem:martingale_sizes} (below) with the kept \(1/d\) factor yields, for any $\varepsilon\in(0,\tfrac12-c_*)$ and large $d$,
\[
\sup_{t\le T}\big(|\mathcal{M}^{(1)}_t|+|\mathcal{M}^{(2)}_t|\big)
 \le  d^{-1/2+c_*+\varepsilon}.
\]
Furthermore,
\[
\frac{1}{d}  \mathrm{tr}\!\big(\bm\Sigma^2  \bm Q_t^{-2}\bm Q_s^{2}\big)
 \le  \frac{1}{d}  \|\bm\Sigma\|_{\op}  \mathrm{tr}\!\big(\bm\Sigma  \bm Q_t^{-2}\bm Q_s^{2}\big)
 \le  \|\bm\Sigma\|_{\op}^2,
\]
so Grönwall’s inequality gives, uniformly in $t\le T$,
\[
|\Delta_t|
 \le  e^{T\|\bm\Sigma\|_{\op}^2}   d^{-1/2+c_*+\varepsilon}.
\]
Since \(e^{(\ln d)^{3/4}}=d^{o(1)}\) and \(T=n/d\le (\ln d)/c_0\), by choosing \(c_0\) large we absorb the resulting subpolynomial factor into \(\varepsilon\). Taking $\varepsilon\le c_*$ yields
\(
|\Delta_t|\le d^{-1/2+2c_*}
\)
for all large $d$. Finally, the event from Lemma~\ref{lem:boundednessofP} and the martingale tail together dominate any $d^{-c_1}$ tail, proving the claim.
\end{proof}

\begin{lemma}
\label{lem:boundednessofP}
Under the assumptions in Section~2.1, the loss is bounded with high probability: there exists a sufficiently large constant $C>0$ such that
\[
\mathbb{P}\!\left( \sup_{t \le T} \mathcal{P}(\bm X_t) \le C e^{C t + (\ln d)^{3/4}} \right) \ge 1 - e^{-c (\ln d)^{3/2}}.
\]
\end{lemma}
\begin{proof}
Let \( u_t := \tfrac{1}{2}\|\bm X_t\|^2 \). Using the HSGD dynamics and Itô’s formula,
\begin{align*}
du_t
&= \bm X_t^\top d\bm X_t + \tfrac{1}{2}   d\langle \bm X \rangle_t \\
&= -\gamma(t)  \bm X_t^\top \nabla \mathcal{R}(\bm X_t)  dt
   + \gamma(t)  \bm X_t^\top \sqrt{\tfrac{1}{d}\!\left(2\mathcal{P}(\bm X_t)\bm\Sigma + \sigma^2 \bm I_d\right)}   d\bm B_t \\
&\qquad + \frac{\gamma^2(t)}{2}  \operatorname{tr}\!\left(\tfrac{1}{d}\!\left(2\mathcal{P}(\bm X_t)\bm\Sigma + \sigma^2 \bm I_d\right)\right) dt .
\end{align*}
Since $\|\bm\Sigma\|_{\op}$ and $\gamma(\cdot)$ are bounded,
\[
\frac{\gamma^2(t)}{2}  \operatorname{tr}\!\left(\tfrac{1}{d}\!\left(2\mathcal{P}(\bm X_t)\bm\Sigma + \sigma^2 \bm I_d\right)\right)
 \le  \frac{\gamma^2(t)}{2}\Big(2  \mathcal{P}(\bm X_t)  \|\bm\Sigma\|_{\op} + \sigma^2\Big).
\]
For the quadratic variation of the martingale part of \(u_t\),
\begin{align*}
d\langle u \rangle_t
&= \gamma^2(t)  \bm X_t^\top \tfrac{1}{d}\!\left(2\mathcal{P}(\bm X_t)\bm\Sigma + \sigma^2 \bm I_d\right)\bm X_t  dt \\
&\le \frac{C}{d}  \|\bm X_t\|^2 \big(\mathcal{P}(\bm X_t)+1\big)  dt
 =  \frac{C}{d}  u_t\big(\mathcal{P}(\bm X_t)+1\big)  dt .
\end{align*}
By Assumptions 1 and 2,
\[
\mathcal{P}(\bm x)
= \tfrac12(\bm x-\tilde{\bm x})^\top \bm\Sigma (\bm x-\tilde{\bm x}) + \tfrac12  \mathbb{E}[w^2]
 \le  C\big(\|\bm x\|^2 + \|\tilde{\bm x}\|^2 + 1\big)
 \le  C\big(u + d^{2c_*}\big),
\]
so
\[
d\langle u \rangle_t
 \le  \frac{C}{d}  u_t\big(u_t + d^{2c_*} + 1\big)  dt
 \le  \frac{C}{d}\Big((u_t+1)^2 + d^{2c_*}\Big)  dt .
\]

Set \( z_t := \ln(1+u_t) - \tfrac{C}{d}  t \). A direct Itô calculation shows \(z_t\) is a supermartingale for $C$ large enough (depending only on model constants), and
\[
d\langle z \rangle_t
= \frac{1}{(1+u_t)^2}   d\langle u \rangle_t
 \le  \frac{C}{d}  dt  +  \frac{C}{d}  \frac{d^{2c_*}}{(1+u_t)^2}  dt
 \le  \frac{C}{d}  \big(1+d^{2c_*}\big)  dt .
\]
Applying Freedman’s Inequality for continuous martingales with threshold $(\ln d)^{3/4}$ yields
\[
\mathbb{P}\!\left(\sup_{t\le T} z_t > (\ln d)^{3/4}\right)
\le \exp\!\big(-c(\ln d)^{3/2}\big),
\]
for all large $d$. On this event,
\[
\ln(1+u_t) \le \tfrac{C}{d}t + (\ln d)^{3/4}
  \Rightarrow  
u_t \le \exp\!\big(\tfrac{C}{d}t + (\ln d)^{3/4}\big) - 1
 \le  C  e^{C t + (\ln d)^{3/4}} .
\]
Finally,
\[
\mathcal{P}(\bm X_t)  \le  C\big(u_t + d^{2c_*}\big)
 \le  C  e^{C t + (\ln d)^{3/4}}  +  C  d^{2c_*}.
\]
\end{proof}

\begin{lemma}
\label{lem:martingale_sizes}
Assume $c_*<\tfrac12$. With probability at least $1-C e^{-(\ln d)^{4/3}}$ (for a constant $C$ depending only on $T$ and model constants),
\[
\sup_{0\le t\le T}\Big(|\mathcal{M}^{(1)}_t|+|\mathcal{M}^{(2)}_t|\Big)
 \le 
d^{-1/2+c_*+\varepsilon}
\]
for any fixed $\varepsilon\in(0,\tfrac12-c_*)$ and all sufficiently large $d$.
\end{lemma}

\begin{proof}
We detail $\mathcal{M}^{(1)}$; the same estimates apply to $\mathcal{M}^{(2)}$.
For fixed $t$, define
\[
M_u^{(1,t)} := \nabla\mathcal{P}(\bm X^{\mathrm{gf}}_{\Gamma(t)})^\top
\bm Q_t^{-1}\!\int_0^u \gamma(s)  \bm Q_s \sqrt{\tfrac{1}{d}\big(2\mathcal{P}(\bm X_s)\bm\Sigma+\sigma^2 \bm I_d\big)}  d\bm B_s.
\]
Then
\[
d\langle M^{(1,t)}\rangle_u
= \gamma^2(u)  \nabla\mathcal{P}(\bm X^{\mathrm{gf}}_{\Gamma(t)})^\top
\bm Q_t^{-1}\bm Q_u  \tfrac{1}{d}\big(2\mathcal{P}(\bm X_u)\bm\Sigma+\sigma^2 \bm I_d\big)\bm Q_u^\top\bm Q_t^{-\top}
\nabla\mathcal{P}(\bm X^{\mathrm{gf}}_{\Gamma(t)})  du.
\]
On the high-probability event of Lemma~\ref{lem:boundednessofP} and using $\|\bm Q_t^{-1}\bm Q_u\|_{\op}\le 1$,
\[
\langle M^{(1,t)}\rangle_t
\le \frac{C}{d}  \|\nabla\mathcal{P}(\bm X^{\mathrm{gf}}_{\Gamma(t)})\|^2\!
\left(e^{CT+(\ln d)^{3/4}} + d^{2c_*}\right)
\le \frac{C}{d}   d^{2c_*}   e^{CT+(\ln d)^{3/4}}.
\]
By Freedman’s inequality for continuous martingales, we have with probability $\ge 1-e^{-c(\ln d)^{3/2}}$,
\[
\sup_{u\le t}|M_u^{(1,t)}|
 \le  C\sqrt{\langle M^{(1,t)}\rangle_t}  (\ln d)^{3/4}
 \le  d^{-1/2+c_*+o(1)}.
\]


Fix $\varepsilon\in(0,\tfrac12-c_*)$ and set a uniform mesh $\mathcal T_h=\{t_k=kh:0\le k\le \lceil T/h\rceil\}$ with $h=d^{-\varepsilon}$.
For each $t_k\in\mathcal T_h$, the preceding Freedman/Bernstein bounds give
\[
\Pr\!\Big(\sup_{u\le t_k}\big(|\mathcal M^{(1)}_u|+|\mathcal M^{(2)}_u|\big)>C  d^{-1/2+c_*+\varepsilon}\Big)
\le e^{-c(\ln d)^{3/2}} .
\]
Since $|\mathcal T_h|\le C  d^{\varepsilon}$, a union bound yields
\[
\Pr\!\Big(\max_{t_k\in\mathcal T_h}\sup_{u\le t_k}\big(|\mathcal M^{(1)}_u|+|\mathcal M^{(2)}_u|\big)
> C  d^{-1/2+c_*+\varepsilon}\Big)\le e^{-c'(\ln d)^{3/2}} .
\]
For $t\in[t_k,t_{k+1}]$, write
\[
\mathcal M^{(1)}_t
= M_t^{(1,t_{k+1})}
+\big(\bm\phi(t)-\bm\phi(t_{k+1})\big)^\top\!\int_0^t \bm r_s  d\bm B_s,
\]
where $\bm\phi(t):=\nabla\mathcal P(\bm X^{\mathrm{gf}}_{\Gamma(t)})  \bm Q_t^{-1}$ and
$\bm r_s:=\gamma(s)\bm Q_s\sqrt{\tfrac1d(2\mathcal P(\bm X_s)\bm\Sigma+\sigma^2\bm I)}$.
By bounded coefficients and gradient-flow regularity, $\bm\phi(\cdot)$ is Lipschitz on $[0,T]$ with some $L_T$ independent of $d$. On the high-probability event from Lemma~\ref{lem:boundednessofP}, the Burkholder–Davis–Gundy inequality gives
\[
\sup_{u\le T}\Big\|\!\int_0^u \bm r_s  d\bm B_s\Big\|\le C  d^{-1/2+c_*+o(1)}.
\]
Hence,
\[
\sup_{t\in[t_k,t_{k+1}]}
   \big|\mathcal M^{(1)}_t - M_t^{(1,t_{k+1})}\big|
   \le L_T h\, d^{-1/2 + c_* + o(1)}
   \le \tfrac{1}{2}\, d^{-1/2 + c_* + \varepsilon},
\]
for all sufficiently large $d$.
The same estimate holds for $\mathcal M^{(2)}$.  
Combining the grid and in-between controls, we obtain
\[
\sup_{t\le T}
   \big(|\mathcal M^{(1)}_t| + |\mathcal M^{(2)}_t|\big)
   \le d^{-1/2 + c_* + \varepsilon},
\]
with probability at least
$1 - e^{-c(\ln d)^{4/3}}$,
after absorbing fixed constants into the $d^{\varepsilon}$ factor.

\end{proof}

\begin{proof}[\textbf{Proof of Theorem~\ref{thm:volterra}}]
    In the proof of Theorem 4.1, we have
    \begin{equation}
    \begin{aligned}
        P_t = \mathcal{P}(\bm {X}^{\mathrm{gf}}_{\Gamma(t)}) +
        \frac{1}{d} \int_{0}^t\gamma^2(s) \operatorname{tr}(\bm \Sigma^2 \bm Q_t^{-2}\bm Q_s^{2})P_sds + 
        \frac{\sigma^2}{2d} \int_{0}^t\gamma^2(s) \operatorname{tr}(\bm \Sigma \bm Q_t^{-2}\bm Q_s^{2})ds.
    \end{aligned}
    \label{eq:expectation_risk_p 2}
    \end{equation}
For the the regularized risk $\mathcal{R}$, let $R_t \stackrel{\text{def}}{=} \mathbb{E}[\mathcal{P}(\bm X_t)]$. Applying the same technique, we derive   
    \begin{equation}
    \begin{aligned}
        R_t = \mathcal{R}(\bm {X}^{\mathrm{gf}}_{\Gamma(t)}) +
        \frac{1}{d} \int_{0}^t\gamma^2(s) \operatorname{tr}(\bm \Sigma (\bm \Sigma + \delta \bm I_d) \bm Q_t^{-2}\bm Q_s^{2})P_sds+
        \frac{\sigma^2}{2d} \int_{0}^t\gamma^2(s) \operatorname{tr}((\bm \Sigma + \delta \bm I_d) \bm Q_t^{-2}\bm Q_s^{2})ds.
    \end{aligned}
    \label{eq:expectation_risk_r}
    \end{equation}     
    Using the identities that $\nabla^2 \mathcal{P}(x) = \bm\Sigma$ and $\nabla^2 \mathcal{R}(x) = \bm\Sigma + \delta \bm I_d$, (\ref{eq:expectation_risk_p 2}) and (\ref{eq:expectation_risk_r}) can be rewritten as
    \begin{equation*}
    P_t = \mathcal{P}(\bm {X}^{\mathrm{gf}}_{\Gamma(t)}) + \int_0^t G(t,s;\nabla^2\mathcal{P})P_s ds + \int_0^t G'(t,s;\nabla^2\mathcal{P})ds
    \end{equation*}
    and
    \begin{equation*}
        R_t = \mathcal{R}(\bm {X}^{\mathrm{gf}}_{\Gamma(t)}) + \int_0^t G(t,s;\nabla^2\mathcal{R})P_s ds + \int_0^t G'(t,s;\nabla^2\mathcal{R})ds.
    \end{equation*}
    where 
    \begin{equation*}
    \begin{cases}
    G(t,s;\bm M) = \frac{\gamma^2(s)}{d}\tr 
    \left( 
    (\bm \Sigma\bm M \exp\left(-2(\bm \Sigma+\delta \bm I_d)(\Gamma(t)-\Gamma(s)) \right) \right)\\
    G'(t,s;\bm M) = \frac{\sigma^2\gamma^2(s)}{2d}\tr 
    \left(  \bm M \exp\left(-2(\bm \Sigma+\delta \bm I_d)(\Gamma(t)-\Gamma(s)) \right) \right)
    \end{cases}
\end{equation*}
Letting 
$\Phi(t,s) = e^{-(\bm \Sigma + \delta \bm I_d)(\Gamma(t) - \Gamma(s))}$,
we have the result stated in the theorem.
\end{proof}

\section{Proofs for Section 4.2}

\begin{lemma}[Conditional Gaussian law after the differentiating update]
\label{lem:conditional-gaussian}
Fix a neighboring pair $\theta=((\bm a,b),(\bm a',b'))$ and a differentiating time $s\in(0,T]$. 
For $i\in\{1,2\}$, let
\[
\bm X_{s^+}^{(i)} = \bm C_i \bm X_{s^-} + \bm c_i + \bm e,
\quad
\bm e \sim \mathcal N\!\Big(\bm 0, \tfrac{\gamma^2(s)}{d^2}\sigma^2 \bm I_d\Big),
\]
where $\bm X_{s^-}\sim\mathcal N(m(s),V(s))$ and $\bm e$ is independent of $\bm X_{s^-}$ and of the future Brownian noise. 
For $t\ge s$, evolve under the linear SDE
\begin{equation}
\label{eq:SDE_2}
    d \bm X_t = -\gamma(t) \nabla \mathcal{R}(\bm X_t) \mathrm{d}t + \gamma(t)\bm Q^{1/2}(t) d \bm B_t,
\quad
\bm Q(t) := \frac{1}{d}\left(2 P_t \bm \Sigma + \sigma^2 \bm I_d\right).
\end{equation}
Then, for any $t\ge s$,
\[
\bm X_t^{(i)} \sim \mathcal N\big(m_i(t;s), V_i(t;s)\big),
\]
where
\begin{equation}
\begin{aligned}
    m_i(t;s) &= \Phi(t,s)  m_i(s^+) 
+ \int_s^t \Phi(t,u)  \gamma(u)\bm\Sigma  \tilde{\bm x}  du,\\
V_i(t;s) &= \Phi(t,s)  V_i(s^+)  \Phi(t,s)^\top
+ \int_s^t \Phi(t,u)  \gamma^2(u)\bm Q(u)  \Phi(t,u)^\top du,
\end{aligned}
\label{eq:mean_var}
\end{equation}
and 
\[
m_i(s^+) = \bm C_i m(s) + \bm c_i, 
\qquad
V_i(s^+) = \bm C_i V(s)\bm C_i^\top + \tfrac{\gamma^2(s)}{d^2}\sigma^2\bm I_d.
\]
The state–transition matrix is $\Phi(t,s) = \exp\!\big(-\!\int_s^t \gamma(u)\bm A  du\big)$ where $ \bm A := \bm\Sigma + \delta \bm I_d$.
For $t < s$, $\bm X_t^{(1)} \stackrel{d}{=} \bm X_t^{(2)}$.
\end{lemma}

\begin{proof}
Since $\bm X_{s^-}$ is Gaussian and $\bm X_{s^+}^{(i)}$ is an affine transformation plus independent Gaussian noise, $\bm X_{s^+}^{(i)}$ is Gaussian with mean 
$m_i(s^+) = \bm C_i m(s) + \bm c_i$ 
and covariance 
$V_i(s^+) = \bm C_i V(s)\bm C_i^\top + \tfrac{\gamma^2(s)}{d^2}\sigma^2\bm I_d$.

For $t \ge s$, the linear SDE has the solution
\[
\bm X_t^{(i)} = 
\Phi(t,s)\bm X_{s^+}^{(i)} 
+ \int_s^t \Phi(t,u)  \gamma(u)\bm\Sigma  \tilde{\bm x}  du
+ \int_s^t \Phi(t,u)  \gamma(u)\bm Q^{1/2}(u)  d\bm B_u.
\]
Denote the stochastic integral by 
$\bm N_t := \int_s^t \Phi(t,u)\gamma(u)\bm Q^{1/2}(u)  d\bm B_u$.
Conditioned on $\bm X_{s^+}^{(i)}$, 
\[
\mathbb E[\bm X_t^{(i)} \mid \bm X_{s^+}^{(i)}]
= \Phi(t,s)\bm X_{s^+}^{(i)} + \int_s^t \Phi(t,u)\gamma(u)\bm\Sigma  \tilde{\bm x}  du.
\]
Taking expectations (law of total expectation) gives
\[
m_i(t;s) = \Phi(t,s)m_i(s^+) 
+ \int_s^t \Phi(t,u)\gamma(u)\bm\Sigma  \tilde{\bm x}  du.
\]
Applying the law of total variance,
\[
\operatorname{Var}(\bm X_t^{(i)})
= \mathbb E[\operatorname{Var}(\bm X_t^{(i)} \mid \bm X_{s^+}^{(i)})]
+ \operatorname{Var}(\mathbb E[\bm X_t^{(i)} \mid \bm X_{s^+}^{(i)}]).
\]
The conditional variance is the variance of $\bm N_t$, which by Itô isometry equals
\[
\operatorname{Var}(\bm N_t) 
= \int_s^t \Phi(t,u)  \gamma^2(u)\bm Q(u)  \Phi(t,u)^\top du.
\]
The second term gives 
$\operatorname{Var}(\Phi(t,s)\bm X_{s^+}^{(i)}) 
= \Phi(t,s)V_i(s^+)\Phi(t,s)^\top$.
Combining these two components yields
\[
V_i(t;s) = \Phi(t,s)V_i(s^+)\Phi(t,s)^\top 
+ \int_s^t \Phi(t,u)\gamma^2(u)\bm Q(u)\Phi(t,u)^\top du.
\]
Because $\bm X_t^{(i)}$ is a sum of affine Gaussian and independent Gaussian noise, 
it remains Gaussian with mean $m_i(t;s)$ and covariance $V_i(t;s)$.
For $t < s$, both trajectories coincide, hence 
$\bm X_t^{(1)} \stackrel{d}{=} \bm X_t^{(2)}$.
\end{proof}

\bigskip
\begin{lemma}[Gaussian $\alpha$--Rényi divergence, \cite{Gil2013RnyiDM}]
\label{lem:Gaussian-Renyi}
Let $P=\mathcal N(\mu_1,\Sigma_1)$ and $Q=\mathcal N(\mu_2,\Sigma_2)$ be two non-degenerate Gaussian distributions. Let $\alpha>1$ and $M_\alpha:=\alpha\Sigma_1+(1-\alpha)\Sigma_2$. 
Given $\alpha\Sigma_2^{-1}+(1-\alpha)\Sigma_1^{-1}$
is positive definite, 
the $\alpha$--Rényi divergence between $P$ and $Q$ is given by
\[
D_\alpha(P\|Q)
=
\frac{\alpha}{2}(\mu_1-\mu_2)^\top M_\alpha^{-1}(\mu_1-\mu_2)
+\frac{1}{2(\alpha-1)}\ln\!\frac{\det(\Sigma_1)^{\alpha}\det(\Sigma_2)^{1-\alpha}}
{\det(M_\alpha)}
.
\]

\end{lemma}

\bigskip
\begin{lemma}[Mixture bound for Rényi divergence]
\label{lem:mixture-bound}
Let $\pi$ be a probability measure on an index set $\mathcal S$, 
and for each $s\in\mathcal S$, let $P_s$ and $Q_s$ be probability distributions on a common measurable space.
Define the mixtures 
\[
P = \int P_s\,\pi(ds), 
\qquad 
Q = \int Q_s\,\pi(ds).
\]
Then, for every $\alpha>1$,
\[
\exp\!\big((\alpha-1)D_\alpha(P\|Q)\big)
\le
\int
\exp\!\big((\alpha-1)D_\alpha(P_s\|Q_s)\big)\,
\pi(ds).
\]
Equivalently,
\[
D_\alpha(P\|Q)
\le
\frac{1}{\alpha-1}\ln\!\int
\exp\!\big((\alpha-1)D_\alpha(P_s\|Q_s)\big)\,
\pi(ds).
\]
\end{lemma}

\begin{proof}
By definition of Rényi divergence,
\[
\exp\!\big((\alpha-1)D_\alpha(P\|Q)\big)
=\int
\!\Big(\!\int p_s(x)\,\pi(ds)\!\Big)^{\!\alpha}
\!\Big(\!\int q_s(x)\,\pi(ds)\!\Big)^{\!1-\alpha}
dx,
\]
where $p_s$ and $q_s$ are the densities of $P_s$ and $Q_s$.
Applying Hölder’s inequality to the convex map $x\mapsto x^{\alpha}$ gives
\[
\Big(\!\int p_s(x)\,\pi(ds)\!\Big)^{\!\alpha}
\!\Big(\!\int q_s(x)\,\pi(ds)\!\Big)^{\!1-\alpha}
\le
\int p_s(x)^{\alpha} q_s(x)^{1-\alpha}\,\pi(ds).
\]
Integrating both sides over $x$ and exchanging the order of integration by Fubini’s theorem yields
\[
\exp\!\big((\alpha-1)D_\alpha(P\|Q)\big)
\le
\int
\exp\!\big((\alpha-1)D_\alpha(P_s\|Q_s)\big)\,
\pi(ds).
\]
\end{proof}

\begin{proof}[\textbf{Proof of Theorem~\ref{thm:rdp-last}}]
Fix $\alpha>1$, $t\in(0,T]$, and a neighboring pair $\theta=((\bm a,b),(\bm a',b'))$.
If the differentiating time $s>t$, both processes evolve under the same SDE (\ref{eq:SDE_2}) up to time $t$, hence 
$\bm X_t^{(1)} \stackrel{d}{=} \bm X_t^{(2)}$ and $D_\alpha(t;s,\theta)=0$.

When $s\le t$, conditioning on the differentiating update at $s$ yields
$\bm X_t^{(i)} \sim \mathcal N(m_i(t;s), V_i(t;s))$ for $i\in\{1,2\}$,
where $m_i(t;s)$ and $V_i(t;s)$ are given in (\ref{eq:mean_var}).
Let $\Delta(t;s)=m_1(t;s)-m_2(t;s)$ and $M_\alpha(t;s)=\alpha V_1(t;s)+(1-\alpha)V_2(t;s)$.
Using the closed form of the Rényi divergence between multivariate Gaussian distributions in Lemma~\ref{lem:Gaussian-Renyi},
\[
D_\alpha(t;s,\theta)
=\frac{\alpha}{2}\Delta(t;s)^\top M_\alpha(t;s)^{-1}\Delta(t;s)
+\frac{1}{2(\alpha-1)}\ln\!\left(
\frac{\det V_1(t;s)^{\alpha}\det V_2(t;s)^{1-\alpha}}{\det M_\alpha(t;s)}
\right),
\]
which gives (\ref{eq:Gauss-Rényi}).

In correspondence to  SGD, $s\sim\mathrm{Unif}(0,T)$, and the released law at time $t$
is a mixture $P=\frac{1}{T}\int_0^T P_s  ds$ and $Q=\frac{1}{T}\int_0^T Q_s  ds$,
where $P_s=\mathcal N(m_1(t;s),V_1(t;s))$ and $Q_s=\mathcal N(m_2(t;s),V_2(t;s))$.
By lemma~\ref{lem:mixture-bound},
\[
\exp\!\big((\alpha-1)D_\alpha(P\|Q)\big)
\le \frac{1}{T}\int_0^T
\exp\!\big((\alpha-1)D_\alpha(P_s\|Q_s)\big)  ds .
\]
For $s>t$, the divergence is zero. Taking the worst case over $\theta\in\mathcal C$,
\[
\exp\!\big((\alpha-1)\varepsilon_\alpha(t)\big)
\le \frac{T-t}{T}+\frac{1}{T}\int_0^t 
\exp\!\big((\alpha-1)\sup_{\theta\in\mathcal C}D_\alpha(t;s,\theta)\big)  ds ,
\]
and taking $\frac{1}{\alpha-1}\ln(\cdot)$ gives (\ref{eq:RDP-last}).
\end{proof}

\bigskip

\begin{proof}[\textbf{Proof of Theorem~\ref{thm:rdp-multi}}]
Let $0<t_1<\cdots<t_J\le T$, fix $\alpha>1$ and a neighboring pair $\theta\in\mathcal C$, and condition on a differentiation time $s\le t_J$.
For each $i\in\{1,2\}$ and $r=1,\dots,J$,
\[
\bm X_{t_r}^{(i)}
=
\Phi(t_r,s)\bm X_{s^+}^{(i)}
+\int_s^{t_r}\!\Phi(t_r,u)  \gamma(u)\bm\Sigma  \tilde{\bm x}  du
+\int_s^{t_r}\!\Phi(t_r,u)  \gamma(u)\bm Q^{1/2}(u)  d\bm B_u .
\]
Define the stacked vector
\(
\mathcal X_{\bm t}^{(i)}
:=(\bm X_{t_1}^{(i)\top},\ldots,\bm X_{t_J}^{(i)\top})^\top\in\mathbb R^{Jd}.
\)
Each $\bm X_{t_r}^{(i)}$ is an affine function of the Gaussian pair $(\bm X_{s^+}^{(i)},  \{ \int_s^{t_r}\cdots d\bm B_u\})$, hence $\mathcal X_{\bm t}^{(i)}$ is jointly Gaussian.
By the law of total expectation and independence of future Brownian noise from $\bm X_{s^+}^{(i)}$,
\[
\mathbb E[\bm X_{t_r}^{(i)}]
=
\Phi(t_r,s)  \mathbb E[\bm X_{s^+}^{(i)}]
+\int_s^{t_r}\!\Phi(t_r,u)  \gamma(u)\bm\Sigma  \tilde{\bm x}  du
=
\Phi(t_r,s)  m_i(s^+)
+\int_s^{t_r}\!\Phi(t_r,u)  \gamma(u)\bm\Sigma  \tilde{\bm x}  du .
\]
Stacking these entries yields
\[
\bm m_i(\bm t;s)
:=
\big(m_i(t_1;s)^\top,\ldots,m_i(t_J;s)^\top\big)^\top,
\quad
m_i(t_r;s)
=
\Phi(t_r,s)  m_i(s^+)
+\int_s^{t_r}\!\Phi(t_r,u)  \gamma(u)\bm\Sigma  \tilde{\bm x}  du.
\]
Write
\(
\bm N_{t_r}:=\int_s^{t_r}\Phi(t_r,u)\gamma(u)\bm Q^{1/2}(u)  d\bm B_u
\)
and note $\mathbb E[\bm N_{t_r}]=\bm 0$.
Using EVVE and Itô isometry,
\begin{align*}
\operatorname{Cov}(\bm X_{t_r}^{(i)},\bm X_{t_j}^{(i)})
&=
\operatorname{Cov}\!\big(\Phi(t_r,s)\bm X_{s^+}^{(i)}+\bm N_{t_r} , 
\Phi(t_j,s)\bm X_{s^+}^{(i)}+\bm N_{t_j}\big)\\
&=
\Phi(t_r,s)  V_i(s^+)  \Phi(t_j,s)^\top
+\operatorname{Cov}(\bm N_{t_r},\bm N_{t_j})\\
&=
\Phi(t_r,s)  V_i(s^+)  \Phi(t_j,s)^\top
+\int_s^{t_r\wedge t_j}\!\Phi(t_r,u)  \gamma^2(u)\bm Q(u)  \Phi(t_j,u)^\top du .
\end{align*}
In particular, when $r\ge j$ this can be written in the propagated form
\begin{equation}
\label{eq:block-cov2}
    \begin{aligned}
        \operatorname{Cov}(\bm X_{t_r}^{(i)},\bm X_{t_j}^{(i)})
& =
\Phi(t_r,t_j)  V_i(t_j;s),
\\
V_i(t_j;s)
&=
\Phi(t_j,s)  V_i(s^+)  \Phi(t_j,s)^\top
+\int_s^{t_j}\!\Phi(t_j,u)  \gamma^2(u)\bm Q(u)  \Phi(t_j,u)^\top du ,
    \end{aligned}
\end{equation}
and symmetrically for $r<j$ by transposition.
Therefore the $Jd\times Jd$ block covariance $\bm V_i(\bm t;s)$ has $d\times d$ blocks
\[
\big[\bm V_i(\bm t;s)\big]_{rj}
=
\begin{cases}
\Phi(t_r,t_j)  V_i(t_j;s), & r\ge j,\\[2pt]
\big(\Phi(t_j,t_r)  V_i(t_r;s)\big)^\top, & r<j.
\end{cases}
\]

Having identified $\mathcal X_{\bm t}^{(i)}\sim\mathcal N(\bm m_i(\bm t;s),\bm V_i(\bm t;s))$, let
\(
\Delta^{[J]}(\bm t;s)=\bm m_1(\bm t;s)-\bm m_2(\bm t;s)
\)
and
\(
\bm M_\alpha^{[J]}(\bm t;s)=\alpha\bm V_1(\bm t;s)+(1-\alpha)\bm V_2(\bm t;s).
\)
The Rényi divergence between these two Gaussians equals
\[
D_\alpha^{[J]}(\bm t;s,\theta)
=
\frac{\alpha}{2}
\Delta^{[J]}(\bm t;s)^\top
\big(\bm M_\alpha^{[J]}(\bm t;s)\big)^{-1}
\Delta^{[J]}(\bm t;s)
+\frac{1}{2(\alpha-1)}
\ln\!\left(
\frac{\det(\bm V_1(\bm t;s))^{\alpha}\det(\bm V_2(\bm t;s))^{1-\alpha}}
{\det(\bm M_\alpha^{[J]}(\bm t;s))}
\right),
\]
which is (\ref{eq:Gauss-Rényi-j}). If $s\in(t_{\ell-1},t_\ell]$ with $t_0=0$, the first $\ell-1$ components coincide across the two processes, so the divergence equals (\ref{eq:Gauss-Rényi-j}) computed on the suffix $(t_\ell,\ldots,t_J)$.
Finally, with $s\sim\mathrm{Unif}(0,T)$, applying the mixture bound for Rényi divergence and noting the zero contribution for $s>t_J$ yields (\ref{eq:RDP-j}).
\end{proof}

\bigskip

\begin{proof}[\textbf{Proof of Theorem~\ref{thm:rdp-avg}}]
Let $0<t_1<\cdots<t_J\le T$ and fix $\alpha>1$, a neighboring pair $\theta\in\mathcal C$, and a differentiation time $s\le t_J$. 
Conditioned on $(s,\theta)$, the stacked vector 
$\mathcal X_{\bm t}^{(i)}=(\bm X_{t_1}^{(i)\top},\ldots,\bm X_{t_J}^{(i)\top})^\top \in \mathbb R^{Jd}$
is Gaussian with mean $\bm m_i(\bm t;s)$ and block covariance $\bm V_i(\bm t;s)$ given in the proof of Theorem~4.4 above.
The averaged vector
$\bar{\bm X}_{\bm t}^{(i)}=\frac{1}{J}\sum_{j=1}^J \bm X_{t_j}^{(i)}$
is a linear image of $\mathcal X_{\bm t}^{(i)}$, hence Gaussian with mean
\[
\bar m_i=\frac{1}{J}\sum_{j=1}^J m_i(t_j;s)
\]
and covariance
\[
\bar V_i=\operatorname{Cov}\!\Big(\tfrac{1}{J}\sum_{r=1}^J \bm X_{t_r}^{(i)},\ \tfrac{1}{J}\sum_{j=1}^J \bm X_{t_j}^{(i)}\Big)
=\frac{1}{J^2}\sum_{r=1}^J\sum_{j=1}^J \Sigma_i(t_r,t_j;s),
\]
where $\Sigma_i(t_r,t_j;s)=\mathrm{Cov}(\bm X_{t_r}^{(i)},\bm X_{t_j}^{(i)})$ as in \eqref{eq:block-cov2}.

Let $\Delta^{\mathrm{avg}}:=\bar m_1-\bar m_2$ and $\bar M_\alpha:=\alpha \bar V_1+(1-\alpha)\bar V_2$.
The closed form for the Rényi divergence between multivariate Gaussian laws gives
\[
D_\alpha^{\mathrm{avg}}(\bm t;s,\theta)
=\frac{\alpha}{2}\Delta^{\mathrm{avg}\top}\bar M_\alpha^{-1}\Delta^{\mathrm{avg}}
+\frac{1}{2(\alpha-1)}\ln\!\left(
\frac{\det(\bar V_1)^{\alpha}\det(\bar V_2)^{1-\alpha}}
{\det(\bar M_\alpha)}
\right),
\]
which is \eqref{eq:Gauss-Rényi-avg-min}.

In correspondence to the shuffled SGD, $s\sim \mathrm{Unif}(0,T)$. For $s>t_J$, the two averaged laws coincide and contribute $1$ inside the exponential moment.
By the mixture bound for Rényi divergence,
\[
\exp\!\big((\alpha-1)\varepsilon_\alpha^{\mathrm{avg}}(\bm t)\big)
\le
\frac{T-t_J}{T}
+\frac{1}{T}\int_0^{t_J}
\exp\!\Big((\alpha-1)\sup_{\theta\in\mathcal C} D_\alpha^{\mathrm{avg}}(\bm t;s,\theta)\Big)  ds .
\]
Taking $\tfrac{1}{\alpha-1}\log$ completes the proof.
\end{proof}


\end{document}